\documentclass{wiley-article}
\usepackage{siunitx}
\usepackage{algorithm}
\usepackage{algorithmic}
\usepackage{xcolor}
\usepackage{subfig}
\usepackage[super]{natbib}
\newtheorem{assumption}{Assumption}
\papertype{Original Article}
\paperfield{Journal Section}

\title{Direct and indirect reinforcement learning}
\author[1]{Yang Guan}
\author[1]{Shengbo Eben Li}
\author[1]{Jingliang Duan}
\author[1]{Jie Li}
\author[1]{Yangang Ren}
\author[1]{Qi Sun}
\author[1]{Bo Cheng}
\affil[1]{School of Vehicle and Mobility, Tsinghua University, Beijing, 100084, China}

\corraddress{Shengbo Eben Li, School of Vehicle and Mobility, Tsinghua University, Beijing, 100084, China}
\corremail{lisb04@gmail.com}
\fundinginfo{This work was supported by International Science \& Technology Cooperation Program of China under 2019YFE0100200, and Tsinghua University-Toyota Joint Research Center for AI Technology of Automated Vehicle.}

\runningauthor{Yang Guan et al., Published in \textit{International Journal of Intelligent Systems}}

\begin{document}
\begin{frontmatter}
\maketitle

\begin{abstract}
Reinforcement learning (RL) algorithms have been successfully applied to a range of challenging sequential decision making and control tasks. In this paper, we classify RL into direct and indirect RL according to how they seek the optimal policy of the Markov decision process problem. The former solves the optimal policy by directly maximizing an objective function using gradient descent methods, in which the objective function is usually the expectation of accumulative future rewards. The latter indirectly finds the optimal policy by solving the Bellman equation, which is the sufficient and necessary condition from Bellman's principle of optimality. We study policy gradient forms of direct and indirect RL and show that both of them can derive the actor-critic architecture and can be unified into a policy gradient with the approximate value function and the stationary state distribution, revealing the equivalence of direct and indirect RL. We employ a Gridworld task to verify the influence of different forms of policy gradient, suggesting their differences and relationships experimentally. Finally, we classify current mainstream RL algorithms using the direct and indirect taxonomy, together with other ones including value-based and policy-based, model-based and model-free.

\end{abstract}
\end{frontmatter}

\section{Introduction}\label{sec.intro}
Reinforcement learning (RL) algorithms have been applied and achieved good performance in a wide variety of challenging domains, including robotics \cite{levine2018learning}, computer games \cite{mnih2015human,vinyals2019grandmaster}, board games \cite{silver2016mastering,silver2017mastering} and many others \cite{sutton2018reinforcement}. Over the last few decades, numerous RL algorithms have appeared in literature on sequential decision-making and optimal control, but most fall into one of two families: value-based RL and policy-based RL, depending on whether a parameterized policy has been learned.

Value-based RL methods involve fitting an action-value function, called the Q-values, that captures the expected return for taking a particular action at a certain state. They do not learn an explicit parameterized policy but derive the optimal action by maximizing the Q-values. Early RL algorithms mostly fall into this category, such as policy iteration, value iteration \cite{puterman1978modified}, Monte Carlo algorithm \cite{sutton1988learning}, and temporal difference (TD) methods \cite{rummery1994line, watkins1992q}. Commonly, they alternatively perform the policy evaluation (PEV) and policy improvement (PIM) procedures until the Q-values converge, where the PEV aims to evaluate the Q-values of the current implicit policy while the PIM is to find a better policy given the updated Q-values.
One of the early breakthroughs in RL is the famous Q-learning (QL) \cite{watkins1992q}, which can learn a deterministic greedy policy while exploring with a stochastic policy, e.g. $\epsilon$-greedy. Its recent extension on approximate functions, deep Q-networks (DQN), may be the most important work in RL in the last decade. DQN adopts neural networks (NN) as Q-values, and introduces the target network and experience replay to obtain competitive performance against human beings in Go and computer games \cite{mnih2015human, silver2016mastering}, kicking off the recent extensive studies on RL. Numerous improvements based on DQN has been made to further boost performance. Double DQN (DDQN) partially addresses the overestimation problem of QL by decoupling selection and evaluation of the bootstrap action \cite{van2016deep}. Prioritized experience replay (PER) replays transitions with larger weights more frequently to learn more efficiently than DQN \cite{schaul2015prioritized}. Dueling DQN uses dueling architecture consists of two streams that represent the value and advantage functions to help to generalize across actions \cite{wang2015dueling}. Instead of estimating Q-values, Bellemare et al. proposed to learn a categorical distribution of discounted returns, which preserves multimodality in value distributions and leads to more stable learning \cite{bellemare2017distributional}. Rainbow combines these independent improvements on the DQN and provides state-of-the-art performance on the Atari 2600 benchmark \cite{hessel2018rainbow}. While QL and its variants learn by one-step bootstraping, Retrace($\lambda$) is the first return-based off-policy control algorithm converging a.s. to $Q^*$ \cite{munos2016safe}. Apart from the aforementioned model-free methods, some studies employ a model in the value-based learning to further enhance its performance. Gu et al. proposed a continuous variant of the Q-learning algorithm called normalized advantage functions (NAF). Specially, they explore the use of learned models to improve the efficiency of their approach \cite{gu2016continuous}. Oh et al. proposed to learn a model of rewards, in contrast to typical model-based RL methods which learn a model of future observations, to augment model-free learning with good results on a number of Atari games \cite{oh2017value}. Holland et al. used a variant of Dyna to learn a model of the environment and generate experience for policy training in the context of Atari games \cite{holland2018effect}. Azizzadenesheli et al. proposed an algorithm called Generative Adversarial Tree Search (GATS) and trained a GAN-based world model along with a Q-function for five Atari games \cite{azizzadenesheli2018sample}.

However, value-based methods would fail in the case of continuous or large discrete action space because finding the greedy policy would become impractical in such action space. To solve this problem, methods that learn a parameterized policy were proposed, called policy-based methods, in which the policy enables actions to be taken without consulting the Q-values. Policy-based methods can learn specific probabilities for taking actions. Besides, they can approach deterministic policies asymptotically and naturally handle continuous action space. The REINFORCE algorithm might be the first attempt of policy-based learning, which adjusts the policy parameter by a Monte-Carlo policy gradient (PG) \cite{williams1992simple}. The study by Sutton et al. formally presented the vanilla PG theorem with a differentiable approximation function, and proved its convergence to a local minimum \cite{sutton2000policy}. Following their studies, the actor-critic algorithm was proposed by Konda and Tsitsiklis, which learns a parameterized value as the baseline to reduce the variance in vanilla policy gradients \cite{konda2000actor}. Both REINFORCE and vanilla PG are on-policy, i.e., training samples are collected under the target policy. To enhance the sample efficiency, off-policy technique is extensively studied, which can learn the target policy using samples from any behavior policies and thus facilitate asynchronous training \cite{mnih2016asynchronous, horgan2018distributed, espeholt2018impala}. Importance Sampling (IS) is a common adopted way for off-policy training \cite{degris2012off}, but nevertheless suffers from large variance. Some works have attempted to use a truncated IS to prevent variance explosion \cite{wang2016sample, gruslys2017reactor}, while others tried to combine on- and off-policy learning, primarily using a interpolated scheme \cite{gu2017interpolated, oh2018self}, with varying degrees of success. So far in the abovementioned methods, the policy must be stochastic, which always has a probability distribution over actions. In 2014, Silver et al. first proposed the deterministic policy gradient (DPG), which chooses a parameterized function to make deterministic decision. The DPG has the ability to accurately estimate
the policy gradient using less samples as it removes the integration over actions \cite{silver2014deterministic}. Its ``deep'' variants, deep DPG (DDPG) and twin delayed DDPG (TD3) first successfully solves more than 20 simulated physics tasks \cite{lillicrap2015continuous, fujimoto2018addressing, d4pg2018}. One challenge of policy-based methods, is how to improve the training stability. The key idea is to avoid changing the policy too much at each step, such as using trust-region constraints or a clipped surrogate loss function \cite{kakade2002natural,schulman2015trust, wu2017scalable, schulman2017proximal}. In addition, policy gradient can also be combined with entropy regularization to increase the ability of policy exploration \cite{haarnoja2017reinforcement, haarnoja2018soft, duan2020addressing}. Duan et al. proposed distributional soft actor-critic (DSAC), which incorporates the entropy measure into the reward to encourage exploration, and updates the policy to minimize the KL-divergence with Boltzmann policy, achieving the state-of-the-art performance on the suite of MuJoCo continuous control tasks. \cite{duan2020addressing}. While the above methods need to have an analytical form of policy gradient, finite difference (FD) is a class of black box optimization algorithms. Although FD methods are not suitable for large capacity policy models like neural networks due to high computation expenses, they have several advantages: it is tolerant of extremely long horizons, and does not need temporal discounting or value function approximation \cite{salimans2017evolution, mania2018simple}. Environment models are usually employed in the policy-based method in two ways. The first one requires a differentiable model, through which the analytical PG can be directly calculated by backpropagation of rewards along a trajectory. A pioneer work is PILCO proposed by Deisenroth and Rasmussen, in which they learned a probabilistic model, and computed an analytical PG over a finite-horizon objective function to learn a deterministic policy \cite{deisenroth2011pilco}. Recently, Dreamer is proposed to handle control with images as input, which encodes observation in the latent space by reconstruction and learns a world model by supervised learning. It can efficiently learn behaviors purely from latent imagination by propagating analytic gradients, making itself a new paradigm of control method through image \cite{DREAMER, CURL, AUGMENTEDATA}.
Another famous example is known as the approximate dynamic programming (ADP). Under its framework a large number of structures are proposed \cite{white1992handbook,prokhorov1997adaptive, heess2015learning}. The other way is to make use of model generated data in the learning process and then use the model-free technique. Feinberg et al. presented model-based value expansion (MVE), which learns the value function using a fixed depth model predictions and updates the policy by DPG \cite{feinberg2018model}. Similar idea can be found in the literature \cite{levine2013guided, buckman2018sample, kurutach2018model, ha2018recurrent, racaniere2017imagination, janner2019trust}.

These two families of RL are distinctive in the algorithm structure but have a few subtle connections. Some works has attempted to draw the exact equivalence between their respect representatives, i.e., QL and vanilla PG, mostly in the entropy regularized framework with a stochastic policy. The very beginning of such exploration is the finding of Boltzmann policy, which demonstrates the QL secretly learns such a policy, even if it has no explicit one \cite{ziebart2010modeling}. O’Donoghue et al. discussed the connection between the fixed points and updates of PG and QL methods. In particular they show that TD-actor-critic \cite{konda2003onactor} is equivalent to expected-SARSA \cite{sutton2018reinforcement}, though the discussion of fixed points is restricted to the tabular setting, and the discussion comparing updates is informal and shows an approximate equivalence. Schulman et al. showed that under appropriate conditions, there exists a precise equivalence between QL and PG in a way that the gradient of the loss function used in n-step QL is equal to the gradient of the loss used in an n-step PG method, including a squared-error term on the value function \cite{schulman2017equivalence}. Nachum et al. developed a new RL algorithm called path consistency learning (PCL) from the relationship between softmax temporal value consistency and policy optimality under entropy regularization, and discovered PCL is in fact the generalization of both PG and QL algorithms given different conditions \cite{nachum2017bridging, nachum2017trust}.

In summary, the current mainstream RL algorithms are categorized by the taxonomy of value-based and policy-based, depending on whether a parameterized policy is incorporated in the algorithm structure. It is explicit but rather superficial for us to understand the behind mechanisms. Besides, the equivalence discussions based on it confine to the specific algorithms, which is precise, but is too trivial to capture the essential differences and correlations between the two classes of RL methods in the big picture. In this paper, we make clear of the question of what makes the distinctions of the two families of RL algorithms from the optimal control perspective. The contributions are emphasized in the following way:

1. Enlightened by the optimal control, we employ a ``direct" and ``indirect" taxonomy, which classifies the RL algorithms by the mechanism they attain the optimal policy. Specifically, the direct RL gets the optimal policy by directly optimizing an objective function concerning the expectation of accumulative future rewards, while the indirect RL turns to solve the sufficient and necessary condition of the optimal solution, i.e., Bellman equation, from which the optimal policy can be derived indirectly. We found from the experiment that the value-based method is actually a special class of indirect RL, where it implies a greedy policy defined on a finite and discrete action space. On the other hand, the policy-based method can be either direct or indirect, depending on its original derivation. Our taxonomy is closer to the basic mechanisms in RL to help us understand algorithms in a deeper and more systematic manner.

2. We study the policy gradient form of both direct and indirect RL, and show that they only differ in terms of the state distribution and the value function, and under certain conditions they both can be unified into a policy gradient with the stationary state distribution and the approximate value function. Besides, the actor-critic architecture can be derived separately from them, revealing the equivalence between the two families of RL algorithms.

3. We verify the theoretical results by a simulation, in which different forms of PG are compared. The results manifest the influence of the state distribution and the value function on the training process, from which the differences and correlations between the direct and indirect RL can be recognized.

4. We classify current mainstream RL algorithms according to direct and indirect taxonomy, together with other taxonomies including value-based and policy-based, model-based and model-free.

The rest of this paper is organized as follows. Section \ref{sec.preliminaries} provides necessary RL preliminaries in terms of value functions and stationary state distribution. Section \ref{sec.dirl} introduces the definition of the direct and indirect RL, studies the form of their policy gradient, and reveals their equivalence under particular conditions. Based on that, the algorithm classification under our taxonomy, together with other taxonomies, is presented. Then in section \ref{sec.convergence}, the convergence results of the direct and indirect RL are established based on existing theorems. Section \ref{sec.experiment} conducts two experiments to verify the connection between our taxonomy and others, as well as the relationship between the direct and indirect RL. Finally, section \ref{sec.conclusion} concludes our findings and discusses the future work.

\section{Preliminary}\label{sec.preliminaries}
We study the standard reinforcement learning setting in which an agent interacts with an environment by observing a state $s$, selecting an action $a$, receiving a reward $r$, and observing the next state $s'$. We model this process with a Markov Decision Process (MDP) $\langle\mathcal{H}, p, \pi, r, d^{0}\rangle$. Here, $\mathcal{H}= \left(\mathcal{S}, \mathcal{A}\right)$ denotes the state and action space and $p: \mathcal{S} \times \mathcal{A} \times \mathcal{S} \rightarrow \mathbb{R}$ is the transition function. Throughout this paper we will assume that $\mathcal{S}$ and $\mathcal{A}$ are finite set and $\mathcal{S}$ has $n$ different states, i.e., $\mathcal{S}=\{s^1, s^2, \dots, s^n\}$. A policy $\pi$ maps a state to a distribution over actions, $r : \mathcal{S} \times \mathcal{A} \times \mathcal{S} \rightarrow \mathbb{R}$ is the reward function, $d^{0} : \mathcal{S} \rightarrow \mathbb{R}$ is the distribution of the initial state $s_0$ and we define $\gamma$ as the discount factor. Next, we will mainly introduce the concept of value function and stationary state distribution in RL, to support the following analysis.

\subsection{Value functions in RL}
In RL, we always seek to find an optimal policy ${\pi}^*$ which can maximize a long-term objective about the expected discounted sum of future rewards over the initial state distribution
\begin{equation}\label{eq.orig_obj}
    \pi^*=\arg\max_{\pi}\mathbb{E}_{s_0, a_1, s_2,\cdots}\left\{\sum_{l=0}^{\infty} \gamma^{l}r_l \right\},
\end{equation}
where $s_0\sim d^0$, $a_k \sim \pi\left(a_k | s_k\right)$, $s_{k+1} \sim p\left(s_{k+1} | s_k, a_k\right)$ and $r_{k}= r\left(s_{k}, a_{k}, s_{k+1}\right), k\ge0$. To measure the value of a state, the state-value function $v^{\pi}(s)$ is accordingly defined as the expected discounted sum of future rewards starting from the state $s$ and thereafter following policy $\pi$,
\begin{equation}\label{eq.def_state_value}
    v^{\pi}(s) =\mathbb{E}_{a_{0}, s_{1}, \cdots}\left\{\sum_{l=0}^{\infty} \gamma^{l}r_l  | s_0=s\right\}
\end{equation}
Similarly, the action-value function $q^{\pi}(s,a)$ is defined to quantify the value of taking action $a$ in state $s$,
\begin{equation}\label{eq.def_action_value}
    q^{\pi}(s,a) =\mathbb{E}_{s_{1}, a_{1}, \ldots}\left\{\sum_{l=0}^{\infty} \gamma^{l} r_l | s_{0}=s, a_{0}=a\right\} = \mathbb{E}_{s_1\sim p}\left\{r_0 + \gamma v^{\pi}(s_1)| s_{0}=s, a_{0}=a\right\}
\end{equation}
where the second equation reveals the relationship between the two kinds value function. It holds from the definition of the state-value function.

By the dynamic programming, we can get the relationship between the state values of adjacent states under arbitrary policies, called self-consistency condition
\begin{equation}
\label{eq.self_consistent}
    v^{\pi}(s) =\mathbb{E}_{a \sim \pi, s' \sim p}\left\{r+ \gamma v^{\pi}(s')\right\},
\end{equation}
The Bellman equation is the self-consistency condition under the optimal policy
\begin{equation}\label{eq.bellman}
v^*(s) = \max_{\pi}\left[ \mathbb{E}_{a \sim \pi, s' \sim p}\left\{r+ \gamma v^*(s')\right\}\right]
\end{equation}
where $v^*$ is the state-value function of the optimal policy. For simplicity, we then rewrite the self-consistency condition and the Bellman equation in a more compact vector form. First, we define the state transition function $p_{\pi}\left(s^{\prime} | s\right) =\sum_{a \in \mathcal{A}} \pi(a | s) p\left(s^{\prime} | s, a\right)$ and denote $\mathbb{E}_{a \sim \pi, s' \sim p}\left\{ r(s, a, s')\right\}$ as $r_{\pi}(s)$. On the basis of that, we build the vector form of value function, reward function and the state transition matrix, shown as follows,
\begin{equation}\label{eq.vector_form}
\begin{aligned}
    \bm{v^{\pi}}&=
    \begin{bmatrix}
    v^{\pi}(s^1) & v^{\pi}(s^2) & \cdots & v^{\pi}(s^n)
    \end{bmatrix}^{\top},\\
    \bm{v^*}&=
    \begin{bmatrix}
    v^*(s^1) & v^*(s^2) & \cdots & v^*(s^n)
    \end{bmatrix}^{\top},\\
    \bm{r_{\pi}}&=
    \begin{bmatrix}
    r_{\pi}(s^1) & r_{\pi}(s^2) & \cdots & r_{\pi}(s^n)
    \end{bmatrix}^{\top},\\
    \bm{p_{\pi}}&=
    \begin{bmatrix}
    p_{\pi}(s^1|s^1)& p_{\pi}(s^2|s^1)& \cdots& p_{\pi}(s^n|s^1)\\
    p_{\pi}(s^1|s^2)& p_{\pi}(s^2|s^2)& \cdots& p_{\pi}(s^n|s^2)\\
    \vdots & \vdots & \ddots & \vdots\\
    p_{\pi}(s^1|s^n)& p_{\pi}(s^2|s^n)& \cdots& p_{\pi}(s^n|s^n)\\
    \end{bmatrix},
\end{aligned}
\end{equation}
Finally, we define the self-consistency operator $\mathcal{T}_{\pi}$ and the Bellman operator $\mathcal{T}_{*}$ as
\begin{equation}\label{eq.operations}
\begin{aligned}
\mathcal{T}_{\pi} \bm{v} &=\bm{r_{\pi}}+\gamma \bm{p_{\pi}} \bm{v}\\
\mathcal{T}_{*}\bm{v} &=\max_{\pi}\left[ \bm{r_{\pi}}+\gamma \bm{p_{\pi}} \bm{v}\right]\\
\end{aligned}
\end{equation}
both of which map from an n-dimensional vector to another n-dimensional vector, i.e., $\mathbb{R}^{n} \rightarrow \mathbb{R}^{n}$. The self-consistency condition and the Bellman equation can therefore be rewritten as
\begin{align}
    \bm{v^{\pi}} &= \mathcal{T}_{\pi} \bm{v^{\pi}}\nonumber\\
    \bm{v^*} &= \mathcal{T}_{*}\bm{v^*}\label{eq.vector_bellman}
\end{align}
Both $\mathcal{T}_{\pi}$ and $\mathcal{T}_{*}$ are contraction mappings with respect to maximum norm, which means $\mathcal{T}_{\pi}$ and $\mathcal{T_*}$ have unique fixed point $\bm{v^{\pi}}$ and $\bm{v^*}$, respectively. Besides, the process $\bm{v^{k+1}} =\mathcal{T}_{\pi} \bm{v^{k}}$ converges to $\bm{v^{\pi}}$, and the process $\bm{v^{k+1}} =\mathcal{T}_{*} \bm{v^{k}}$ converges to $\bm{v^*}$.

Large-scale MDPs have numerous states and/or actions to be dealt with. Learning the value function for each individual state appears to be impractical. A empirical way is to utilize function approximation technique to generalize from seen states to unseen states. Specially, state-value function $v^{\pi}(s)$ is estimated by an estimate function $V(s, w)$ parameterized by $w\in \mathbb{R}^m$, $V_w$ or $V_w(s)$ for short. Besides, we approximate the policy by $\pi(a | s, \theta)$ parameterized by $\theta \in \mathbb{R}^m$, $\pi_\theta$ or $\pi_\theta(a|s)$ for short. Since the tabular can be regarded as a special case of the parameterized function, we will mainly discuss how to obtain the optimal policy function $\pi_{\theta^*}$ and the value function $V_{w^*}$. Same as $\bm{v^{\pi_{\theta}}}$, we use $\bm{V_w}$ to denote the vector form of the approximate value function.

\subsection{Stationary state distribution}
A distribution $d: \mathcal{S}\rightarrow\mathbb{R}$ is called a stationary state distribution if it remains unchanged in the Markov chain as time progresses, i.e.,
\begin{equation}
\label{eq.sdp}
    \bm{d}=\bm{p_{\pi}}^{\top}\bm{d},
\end{equation}
where $\bm{d} = [d(s^1)\ d(s^2)\ \cdots\ d(s^n)]^{\top}$. We denote the stationary state distribution of a Markov chain as $d^{\pi}$. According to the properties of Markov chain \cite{ross1996stochastic}, given policy $\pi$, there exists a unique stationary state distribution $d^{\pi}$ if the Markov chain generated by $p_{\pi}$ is indecomposible, nonperiodic and positive-recurrent.
\begin{assumption}\label{myassumption1}{The Markov chain generated by $p_{\pi}$ is indecomposible, nonperiodic and positive-recurrent.}
\end{assumption}
We will assume Assumption \ref{myassumption1} holds throughout this paper. By the
property of Markov chains, the state distribution always becomes stationary as time goes by.

\section{Direct RL and indirect RL}\label{sec.dirl}
In this section, we will formally give the definition of direct and indirect RL, and analyze the form of their policy gradients, respectively. Besides, we will unify them under the actor-critic architecture and discuss their equivalence. Finally, by the proposed taxonomy, mainstream RL algorithms in recent years are classified against other taxonomies.

In the field of optimal control, there are mainly two classes of numerical method to obtain the optimal solution, named by direct and indirect method. The direct optimal control transforms the problem into a nonlinear programming problem, where the controls along the trajectory are parameterized from a finite dimensional space of control function. The problem is then solved by the usual optimization methods. In contrast, the indirect optimal control aims to formulate the necessary conditions of the optimal solution from the calculus of variations or the maximum principle. Then the controls can be obtained analytically or numerically by solving the necessary conditions. Similar to optimal control, RL is also a method of trajectory optimization. Therefore, we make an analogy between RL and optimal control to introduce a similar direct and indirect taxonomy in the field of RL, and provide insight for the connection of the two families of RL algorithm. The direct and indirect RL are defined as follows,
\begin{definition}\label{mydefinition1}
(Direct RL). Direct RL finds the optimal policy by directly maximizing the state-value function, for $\forall s \in \mathcal{S}$.
\end{definition}
\begin{definition}\label{mydefinition2}
(Indirect RL). Indirect RL finds the optimal policy by solving the necessary and sufficient condition from Bellman's principle of optimality, e.g., the Bellman's equation, for $\forall s \in \mathcal{S}$.
\end{definition}
Different from the value-based and policy-based taxonomy, both the direct and indirect RL could have a parameterized policy. In the following analysis, we solve their policies from their definitions by gradient descent method respectively and explore the difference of their policy gradient.

\subsection{Direct RL}
By the definition, direct RL seeks to find $\pi_\theta$ which maximizes the state-value function for every state in $\mathcal{S}$. However, due to the limited fitting ability of the approximation function, current direct RL algorithms usually maximize the expected value over the initial state distribution
\begin{equation}\label{eq.direct_objective_function}
    J(\theta)=\mathbb{E}_{s_0 \sim d^{0}}\left\{v^{\pi_{\theta}}(s_0)\right\}=\sum_{s_0} d^{0}(s_0) v^{\pi_{\theta}}(s_0).
\end{equation}
By policy gradient theorem \cite{sutton2000policy}, the update gradient for the policy function is
\begin{equation}
\begin{aligned}
\nabla_{\theta}J(\theta) &=\nabla_{\theta} \sum_{s_{0}} d^{0}\left(s_{0}\right) v^{\pi}\left(s_{0}\right)\\
& =\sum_{s} \sum_{t=0}^{\infty} \gamma^{t} d^{t}\left(s |\pi_{\theta}\right) \sum_{a} \nabla_{\theta} \pi_{\theta}(a | s) q^{\pi_{\theta}}(s, a),
\end{aligned}  
\end{equation}
where
\begin{equation}\label{eq.t_step_state_distribution}
\begin{aligned}
     d^{t}\left(s |\pi_{\theta}\right)&=\sum_{s_0} d^0(s_0)  \sum_{a_0} \pi_{\theta}(a_0|s_0)\sum_{s_1} p(s_1|s_0,a_0)\ldots \sum_{a_{t-1}} \pi_{\theta}(a_{t-1}|s_{t-1}) p(s|s_{t-1},a_{t-1})\\
     &=\sum_{s_0} d^0(s_0) p_{\pi_{\theta}}^{t}\left(s | s_{0}\right)
\end{aligned}
\end{equation}
is the state distribution at time $t$ starting from the initial state distribution $d^0$ and then following $\pi_{\theta}$. We denote $p_{\pi_{\theta}}^{t}\left(s | s_{0}\right)$, called $t$-step transition function, as the probability of transferring from $s_0$ to $s$ by $t$ steps following the policy $\pi_{\theta}$. Defining the discounted visiting frequency (DVF) 
\begin{equation}\label{eq.DVF}
d^{\gamma}(s|\pi_{\theta})=\sum_{t=0}^{\infty} \gamma^{t} d^{t}\left(s |\pi_{\theta}\right),
\end{equation}
then the direct policy gradient can be expressed as
\begin{equation}\label{eq.direct_policy_gradient}
\begin{aligned}
    \nabla_{\theta}J(\theta)&=\sum_{s} d^{\gamma}(s|\pi_{\theta}) \sum_{a} \nabla_{\theta} \pi_{\theta}(a | s) q^{\pi_{\theta}}(s, a)\\
    &   =\sum_{s} d^{\gamma}(s|\pi_{\theta}) \sum_{a} \nabla_{\theta} \pi_{\theta}(a | s) \sum_{s'} \left[r(s,a)+\gamma v^{\pi_{\theta}}(s')\right]\\
    &   = \mathbb{E}_{s \sim d^{\gamma} , a \sim \pi_{\theta}, s' \sim p}\left\{\nabla_{\theta} \log \pi_{\theta}(a | s)\left[r\!+\!\gamma v^{\pi_{\theta}}(s')\right]\right\}
\end{aligned}
\end{equation}
where the second equation holds from the relationship between the action-value and state-value function. Core procedure of direct RL is shown in Algorithm \ref{alg:vanilla}.

We notice from \eqref{eq.direct_policy_gradient} that the DVF shows up in the state distribution of the PG, however, the properties of the DVF are not clear, which makes it difficult to approximate the PG in practice. Current methods usually substitute it with a state set sampled by $\pi_{\theta}$, but that would cause the bais in the PG \cite{thomas2014bias}. To make clear of the DVF, in the following, we focus on analyzing its properties. We find this is rather important for comparison with the indirect policy gradient.
\begin{algorithm}[tb]
   \caption{Direct reinforcement learning}
   \label{alg:vanilla}
\begin{algorithmic}
   \STATE {\bfseries Initialize:} $\theta$
   \REPEAT
   \STATE $\theta \gets \textbf{Optimizer}(\nabla_{\theta}J(\theta), \theta)$
   \UNTIL{Convergence}
\end{algorithmic}
\end{algorithm}
In practical applications, it is usually intractable to obtain the analytical form of the DVF for each given $\pi_{\theta}$ because of the unknown environment dynamics. And it is also non-trivial to obtain samples that obey it. But from the assumption \ref{myassumption1} we know that samples that obey the stationary state distribution under $\pi_{\theta}$ can be easily accessed as long as we run the policy until the states become stationary. Therefore, we aim to find the relationship between the DVF and the stationary state distribution. We make the following two propositions.
\begin{proposition}\label{myproposition1}
When $d^0(s)=d^{\pi}(s)$ holds for $\forall s\in \mathcal{S}$, then
\begin{equation}
    \begin{aligned}
    d^{\gamma}(s|\pi) 
    =\frac{1}{1-\gamma} d^{\pi}(s).
    \end{aligned}
\end{equation}
\end{proposition}

\begin{proof}
Similar to \eqref{eq.vector_form}, define $\bm{d^0}$, $\bm{d^{\pi}}$, $\bm{d^{t|\pi}}$, $\bm{d^{\gamma|\pi}}$ as the vector form of state distributions $d^0(\cdot)$, $d^{\pi}(\cdot)$, $d^t(\cdot|\pi)$ and $d^{\gamma}(\cdot|\pi)$, and define the $t$-step transition matrix $\bm{p^t_{\pi}}$ using $p^t_{\pi}(\cdot|\cdot)$ accordingly. From \eqref{eq.sdp} and \eqref{eq.t_step_state_distribution}, when $\bm{d^0}=\bm{d^{\pi}}$,
\begin{equation}\label{eq.prop1_stept}
\bm{d^{t|\pi}} = (\bm{p^t_{\pi}})^{\top} \bm{d^0} = (\bm{p_{\pi}}^{\top})^t \bm{d^0}=(\bm{p_{\pi}}^{\top})^t \bm{d^{\pi}}=\bm{d^{\pi}}, \forall t\ge 0,
\end{equation}
where the first equation holds by \eqref{eq.t_step_state_distribution}. The second one holds from the relationship between the $t$-step transition matrix $\bm{p^t_{\pi}}$ and the $1$-step transition matrix $\bm{p_{\pi}}$, i.e., $(\bm{p^t_{\pi}})^{\top} = (\bm{p_{\pi}^{\top}})^t$. The third equation is from our condition. And the last one holds from the definition of stationary state distribution \eqref{eq.sdp}. After that, from the definition of DVF \eqref{eq.DVF},
\begin{equation}
\bm{d^{\gamma|\pi}} =\sum_{t=0}^{\infty} \gamma^{t} \bm{d^{t|\pi}}
=\sum_{t=0}^{\infty} \gamma^{t}\bm{d^{\pi}} 
=\bm{d^{\pi}} \sum_{t=0}^{\infty} \gamma^{t}
=\frac{1}{1-\gamma} \bm{d^{\pi}}.
\end{equation}
where the first equation is the DVF definition, the second equation is the conclusion from the equation \eqref{eq.t_step_state_distribution}, the third one is by the time-independent property of the stationary state distribution, and the last equation holds by the character of geometric series.
\end{proof}

Proposition \ref{myproposition1} tells us that the DVF is proportional to the stationary state distribution, given that the initial state distribution is set to be the stationary state distribution. To give their relationship under arbitrary initial state distribution, we first introduce the following lemma about the $t$-step transition function $p_{\pi}^{t}(s|s_0)$, which is also the state distribution in time step $t$ starting from $s_0$ and then following $\pi$.
\begin{lemma}\label{mylemma1}
If the one step state transition function $p_{\pi}$ of a policy $\pi$ satisfies the Assumption \ref{myassumption1}, then $p_{\pi}^{t}\left(s | s_{0}\right)$ converges to the stationary state distribution of $\pi$ as time goes by, i.e.,
\begin{equation}\label{eq.lemma1_1}
    \lim_{t \rightarrow \infty}p_{\pi}^{t}\left(s | s_{0}\right) = d^{\pi}(s), \forall s,s_0 \in \mathcal{S},
\end{equation}
In addition, the average of the first $N$ timesteps of $p_{\pi}^{t}\left(s | s_{0}\right)$ also converges to the stationary state distribution of $\pi$ \cite{ross1996stochastic}, i.e.,
\begin{equation}\label{eq.lemma1_2}
    \lim_{N \rightarrow \infty} \frac{\sum_{t=1}^{N} p_{\pi}^{t}\left(s | s_{0}\right) }{N} = d^{\pi}(s), \forall s,s_0 \in \mathcal{S}.
\end{equation}
\end{lemma}
The lemma is from the field of random process, which revels the relationship between the stationary state distribution and the $t$-step state distribution from arbitrary initial state in the long term. That means the state distribution in the long term is independent with the initial state. So in the following Proposition, we try to introduce more long-term considerations in the DVF by diminishing the effect of the discount factor.

\begin{proposition}\label{myproposition2}
When $\gamma$ approaches 1,
\begin{equation}\label{DVF:2}
    \begin{aligned}
\lim _{\gamma \rightarrow 1}(1-\gamma) d^{\gamma}(s|\pi) = d^{\pi}(s), \forall s\in \mathcal{S}.
    \end{aligned}
\end{equation}
\end{proposition}
\begin{proof}
\begin{align}
\lim _{\gamma \rightarrow 1}(1-\gamma) d^{\gamma}(s|\pi) & = \lim _{\gamma \rightarrow 1} \frac{(1-\gamma) d^{\gamma}(s|\pi)}{(1-\gamma) \sum_{s} d^{\gamma}(s|\pi)}\nonumber\\
& = \lim _{\gamma \rightarrow 1} \lim _{N \rightarrow \infty} \frac{\sum_{t=0}^{N} \gamma^{t} d^{t}\left(s |\pi\right)}{\sum_{s} \sum_{t=0}^{N} \gamma^{t} d^{t}\left(s |\pi\right)}\nonumber\\
& =\lim _{N \rightarrow \infty} \frac{\sum_{t=0}^{N} d^{t}\left(s |\pi\right)}{\sum_{t=0}^{N} \sum_{s} d^{t}\left(s |\pi\right)}\nonumber\\
& =\lim _{N \rightarrow \infty} \frac{\sum_{t=0}^{N} d^{t}\left(s |\pi\right)}{N+1}\nonumber\\
& = \lim _{N \rightarrow \infty} \frac{\sum_{t=0}^{N} \sum_{s_{0}} d^{0}\left(s_{0}\right) p_{\pi}^{t}\left(s | s_{0}\right)}{N+1}\nonumber\\
& = \sum_{s_{0}} d^{0}\left(s_{0}\right) \lim _{N \rightarrow \infty} \frac{\sum_{t=0}^{N} p_{\pi}^{t}\left(s | s_{0}\right)}{N+1}\nonumber\\
& = \sum_{s_{0}} d^{0}\left(s_{0}\right) d^{\pi}(s)\nonumber\\
& =d^{\pi}(s)\nonumber.
\end{align}
where the first equation holds because
\begin{equation}
\nonumber
(1-\gamma)\sum_{s}d^{\gamma}(s|\pi) = (1-\gamma)\sum_{s} \sum_{t=0}^{\infty}\gamma^t d^t(s|\pi) = (1-\gamma) \sum_{t=0}^{\infty}\gamma^t\sum_{s} d^t(s|\pi)=(1-\gamma) \sum_{t=0}^{\infty}\gamma^t \cdot 1 = 1
\end{equation}
The second equation is to substitute DVF with its definition, i.e., $d^{\gamma}(s|\pi) = \lim_{N\rightarrow\infty}\sum_{t=0}^N\gamma^t d^t(s|\pi)$. The third equation interchanges the two limiting operations, which is guaranteed by the Moore-Osgood theorem. Then the fourth equation calculates the summation in the denominator because the sum of $t$-step state distribution over state space equals to 1 for all $t\ge0$. In the fifth equation, we further span the $t$-step state distribution using the $t$-step transition function $p^t_{\pi}$, whose properties are exhibited in the Lemma \ref{mylemma1}. The sixth equation interchanges the limiting operation and a sum over a finite set, which holds obviously. Finally, the last two equations hold by the equation \eqref{eq.lemma1_2} described in the Lemma \ref{mylemma1}.
\end{proof}

By Proposition \ref{myproposition2}, we can see that when we introduce more long-term effects in the DVF, i.e., letting the discount factor approach 1, the DVF is also proportional to the stationary state distribution no matter what the initial state distribution is. To the best knowledge of us, we are the first to draw the theoretical connection between the DVF and the stationary state distribution.
With Proposition \ref{myproposition1} and \ref{myproposition2}, the direct policy gradient becomes
\begin{equation}\label{eq.direct_pg}
    \begin{aligned}
    &\nabla_{\theta} J(\theta) \approx \frac{1}{1-\gamma} \sum_{s} d^{\pi_{\theta}}(s) \sum_{a} \nabla_{\theta} \pi_{\theta}(a | s) q^{\pi_{\theta}}(s, a)\\
    &    \propto \mathbb{E}_{s \sim d^{\pi_{\theta}} , a \sim \pi_{\theta}}\left\{\nabla_{\theta} \log \pi_{\theta}(a | s) q^{\pi_{\theta}}(s, a)\right\}\\
    &    = \mathbb{E}_{s \sim d^{\pi_{\theta}} , a \sim \pi_{\theta}, s' \sim p}\left\{\nabla_{\theta} \log \pi_{\theta}(a | s)\left[r\!+\!\gamma v^{\pi_{\theta}}(s')\right]\right\}
    \end{aligned}
\end{equation}
where the last equation holds from the relationship between the action-value function and the state-value function \eqref{eq.def_action_value}. These properties of the DVF, we will see later, are the key to formulate the equivalence of the direct and indirect policy gradient. An obvious issue in the direct PG is that the true value function $v^{\pi_{\theta}}$ is not accessible. Although we can always use Monte Carlo method to estimate it as in the REINFORCE algorithm, we can also learn an approximate value function anyway to derive the actor-critic algorithm.
The approximate value function and the policy are regarded as the critic and the actor, which are updated alternatively until convergence. The objective of the critic updating is usually chosen as mean squared error, i.e.
\begin{equation}
    \begin{aligned} 
    J(w) = \mathbb{E}_{s\sim d^0}\left\{(G^{\pi_{\theta}}(s)-V(s,w))^2 \right\},
    \end{aligned}
\end{equation}
where $G^{\pi_{\theta}}(s)$ is the update target, which is used to approximate the true value function and can be constructed by Monte Carlo methods. It calculates the target using the definition of state-value function \eqref{eq.def_state_value}, i.e.,
\begin{equation}\label{eq.value_target1}
    G^{\pi_{\theta}}(s)=\lim_{T \rightarrow \infty}\left[\sum_{l=t}^{T} \gamma^{l-t} r_l|s_t= s\right],
\end{equation}
The temporal difference methods can also be employed to construct the value target from the results of the self-consistency condition \eqref{eq.self_consistent}, i.e.,
\begin{equation}\label{eq.value_target2}
    G^{\pi_{\theta}}(s)=r(s, a, s')+ \gamma V(s',w_{\text{old}})
\end{equation}
Then the gradient of the critic becomes
\begin{equation}\label{eq.critic_gradient}
    \begin{aligned} 
    \nabla_{w}J(w) \propto - \mathbb{E}_{s\sim d^{\pi_{\theta}}}\left\{(G^{\pi_{\theta}}(s)-V(s,w))\nabla_w V(s, w) \right\}
    \end{aligned}
\end{equation}
By using the value function approximation in the policy gradient estimation, the actor-critic architecture can be derived from direct RL, which is shown in Figure \ref{direct_to_actor_critic}.

\begin{figure}[ht]
\begin{center}
\centerline{\includegraphics[width=6cm]{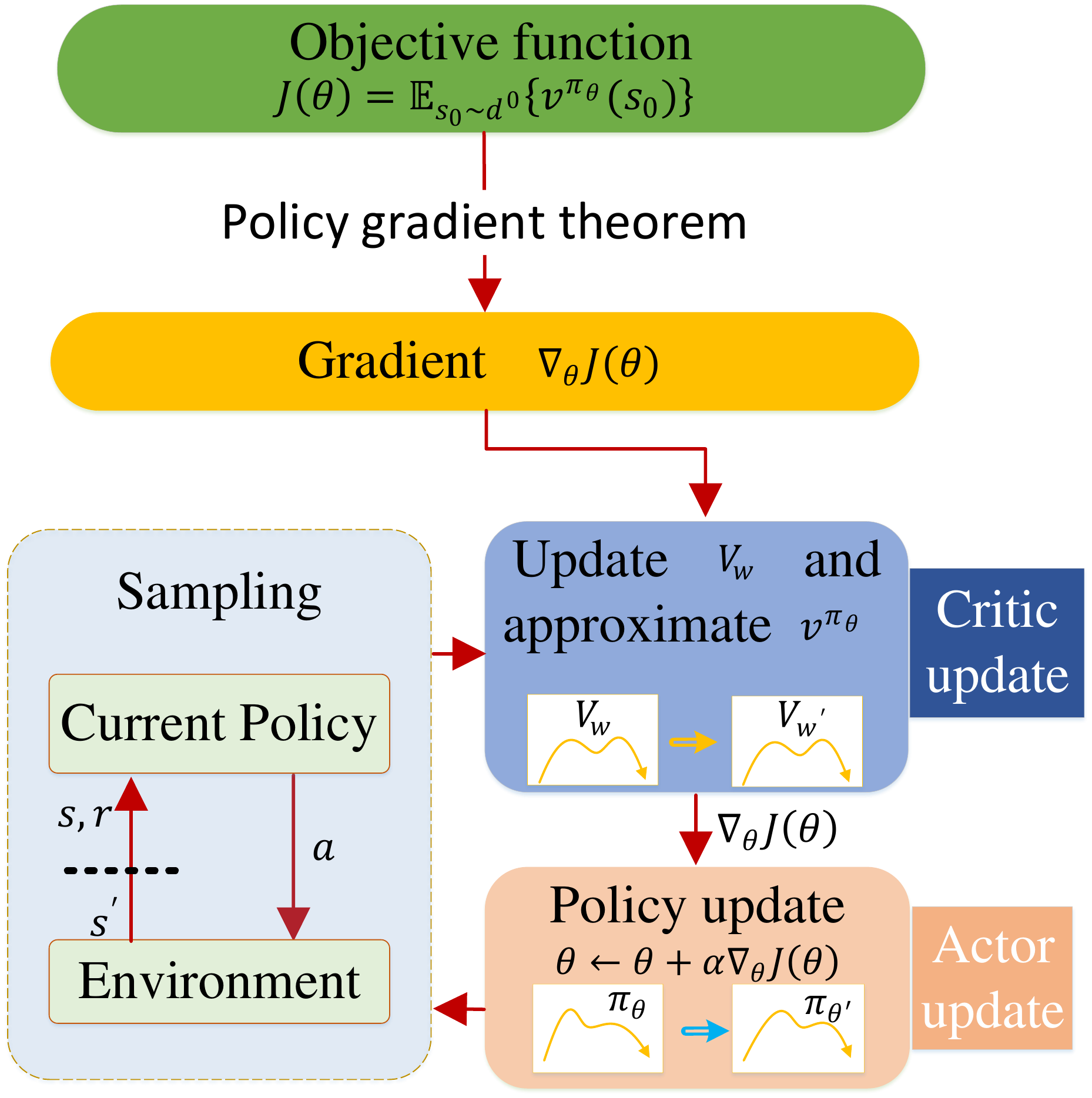}}
\caption{Actor-critic architecture derived from direct RL.}
\label{direct_to_actor_critic}
\end{center}
\end{figure}

\subsection{Indirect RL}
Indirect RL seeks to find the solution of the optimality condition, which is the Bellman equation \eqref{eq.vector_bellman}, and then acquires the optimal policy $\pi^*$ indirectly using the solved optimal value function $v^*$ by the following equation
\begin{equation}
    \pi^* = \arg\max_{\pi}\left[\bm{r_{\pi}}+\gamma \bm{p_{\pi}} \bm{v^*} \right]
\end{equation}
To get the solution of the Bellman equation, the policy iteration algorithm, a typical indirect method in the tabular case, starts with an arbitrary policy $\pi_0$ and generates a sequence of new policies $\pi_1, \pi_2, \cdots$ by implementing the PEV and PIM steps alternatively. Given a policy $\pi$, the PEV step aims to compute $\bm{v^{\pi}}$, which is the only fixed point of the self-consistency condition \eqref{eq.self_consistent}, that is, the solution of the equation
\begin{equation}
    \mathcal{T}_{\pi} \bm{v} = \bm{v}
\end{equation}
By the properties of the fixed policy iteration, the solution can be obtained by keeping performing $\bm{v}=\mathcal{T}_{\pi} \bm{v}$ until its convergence. On the other hand, the PIM step is in charge of computing an optimal policy $\pi'$ based on the value function computed in the last iteration, i.e., finding a new policy that satisfies
\begin{equation}
    \mathcal{T}_{\pi'} \bm{v^{\pi}} = \mathcal{T}_{*} \bm{v^{\pi}}.
\end{equation}
From the theory of dynamic programming, it is guaranteed that the algorithm terminates with the optimal value function $\bm{v^*}$ and the optimal policy $\pi^{*}$.

Inspired by the theoretical results in the tabular value and policy case, we derive a more general indirect algorithm with function approximations, from which the indirect policy gradient is developed. Given a parameterized policy $\pi_{\theta}$, the PEV step aims to find value function parameters $w$ that satisfy $\mathcal{T}_{\pi_{\theta}}\bm{V_w}=\bm{V_w}$ by minimizing the distance between the update target $\mathcal{T}_{\pi_{\theta}}\bm{V_{w_{\text{old}}}}$ and $\bm{V_w}$ iteratively, where the distance is chosen as a Euclidean norm weighted by $d^0$, denoted by $\lVert \bm{x} \rVert_{d^0}$, where
$\lVert \bm{x} \rVert_{d^0}=\sqrt{\sum_{i=1}^n d^0(s^i) x_i^2}$. Therefore the PEV objective becomes
\begin{equation}
    \begin{aligned} 
    J(w) &= \lVert \mathcal{T}_{\pi_{\theta}}\bm{V_{w_{\text{old}}}} - \bm{V_w} \rVert_{d^0}^2\\
    &\approx \mathbb{E}_{s\sim d^0}\left\{(G^{\pi_{\theta}}(s)-V(s, w))^2 \right\},
    \end{aligned}
\end{equation}
where we use \eqref{eq.value_target2} in $G^{\pi_{\theta}}(s)$ to approximate the effect of self-consistency operator $\mathcal{T}_{\pi_{\theta}}$. The value gradient is the same as that in \eqref{eq.critic_gradient}.

Similarly, the PIM step seeks to minimize the distance between $\mathcal{T}_{\pi_{\theta}} \bm{V_w}$ and $\mathcal{T}_{*} \bm{V_w}$. However, here the distance cannot be chosen as the weighted Euclidean norm because we cannot estimate the Bellman operator $\mathcal{T}_*$ by samples as the self-consistency operator. So instead, we select a weighted Manhattan norm $\lVert \bm{x}\rVert_{d^0}$ as the distance defined as $\lVert \bm{x}\rVert=\sum_{i=1}^n d^0(s^i)|x_i|$. The policy objective function therefore becomes to minimize
\begin{equation}
\begin{aligned}
    J(\theta) &= \lVert \mathcal{T}_{\pi_{\theta}} \bm{V_w} - \mathcal{T}_{*} \bm{V_w} \rVert_{d^0}\\
    &= \mathbb{E}_{s\sim d^0}\left\{\left|(\mathcal{T}_{\pi_{\theta}} \bm{V_w})(s)-(\mathcal{T}_{*} \bm{V_w})(s)\right|\right\},
\end{aligned}
\end{equation}
where $(\mathcal{T}_{\pi_{\theta}} \bm{V_w})(s)$ and $(\mathcal{T}_{*} \bm{V_w})(s)$ denote the vector elements correspond to the state $s$. From the definition of the Bellman operator $\mathcal{T}_*$, it is obvious to see that $\mathcal{T}_{\pi_{\theta}}\bm{V_w}(s) \leq \mathcal{T}_{*} \bm{V_w}(s), \forall s\in \mathcal{S}$, and $\mathcal{T}_{*} \bm{V_w}(s)$ is a constant irrelevant to $\theta$. Therefore, we can remove the absolute operator and instead maximize the objective
\begin{equation}\label{PIM}
\begin{aligned}
    J(\theta)&=\mathbb{E}_{s\sim d^0}\left\{(\mathcal{T}_{\pi_{\theta}} \bm{V_w})(s)\right\} \\
    &=\mathbb{E}_{s\sim d^0}\left\{\mathbb{E}_{a\sim \pi_{\theta}(\cdot|s), s'\sim p(\cdot|s,a)}\left\{r+\gamma V(s', w)\right\}\right\}\\
    &=\mathbb{E}_{s\sim d^0, a\sim\pi_{\theta}, s'\sim p}\left\{r+\gamma V(s', w)\right\}\\
    \end{aligned}
\end{equation}
where the second equation holds by the definition of the self-consistency operator in \eqref{eq.operations}. Using the likelihood ratio trick, we can compute the indirect policy gradient, as shown below
\begin{equation}\label{eq.indirect_PG}
\begin{aligned}
    \nabla_{\theta} J(\theta)&=\nabla_{\theta}\mathbb{E}_{s\sim d^0, a\sim\pi_{\theta}, s'\sim p}\left\{r+\gamma V(s', w)\right\}\\
    &=\sum_{s}d^0(s)\sum_{a}\nabla_{\theta}\pi_{\theta}(a|s)\sum_{s'}p(s'|s,a)[r + \gamma V(s',w)]\\
    &=\sum_{s}d^0(s)\sum_{a}\pi_{\theta}(a|s)\sum_{s'}p(s'|s,a)\frac{\nabla_{\theta}\pi_{\theta}(a|s)}{\pi_{\theta}(a|s)}[r + \gamma V(s',w)]\\
    &=\sum_{s}d^0(s)\sum_{a}\pi_{\theta}(a|s)\sum_{s'}p(s'|s,a)\nabla_{\theta}\log\pi_{\theta}(a|s)[r + \gamma V(s',w)]\\
    &=\mathbb{E}_{s \sim d^{0}, a \sim \pi_{\theta}, s^{\prime} \sim p}\left\{\nabla_{\theta} \log \pi_{\theta}(a | s)\left[r+\gamma V\left(s^{\prime}, w\right)\right]\right\}
\end{aligned}
\end{equation}
where the first equation is from \eqref{PIM}. The second equation holds by expanding the expectation notation. The third one is the one uses the likelihood ratio trick, where we multiply the action probability $\pi_{\theta}(a|s)$ term and divide by it to construct the gradient of the logarithm function $\nabla_{theta}\log\pi_{\theta}(a|s)$ shown in the fourth equation. The last equation holds naturally by the notation of the expectation. The whole process of the indirect RL is shown in algorithm \ref{alg:Approximate policy iteration}.
Similar to direct RL, the actor-critic architecture can also be derived from indirect RL. Specifically, the PEV procedure in the indirect RL can be regarded as the update of the critic, while the PIM procedure can be considered as the update of the actor. The gradients of the critic and the actor are displayed in \eqref{eq.critic_gradient} and \eqref{eq.indirect_PG}. The actor-critic architecture derived from indirect RL is illustrated in Figure \ref{indirect_to_actor_critic}.

\begin{figure}[ht]
\begin{center}
\centerline{\includegraphics[width=6cm]{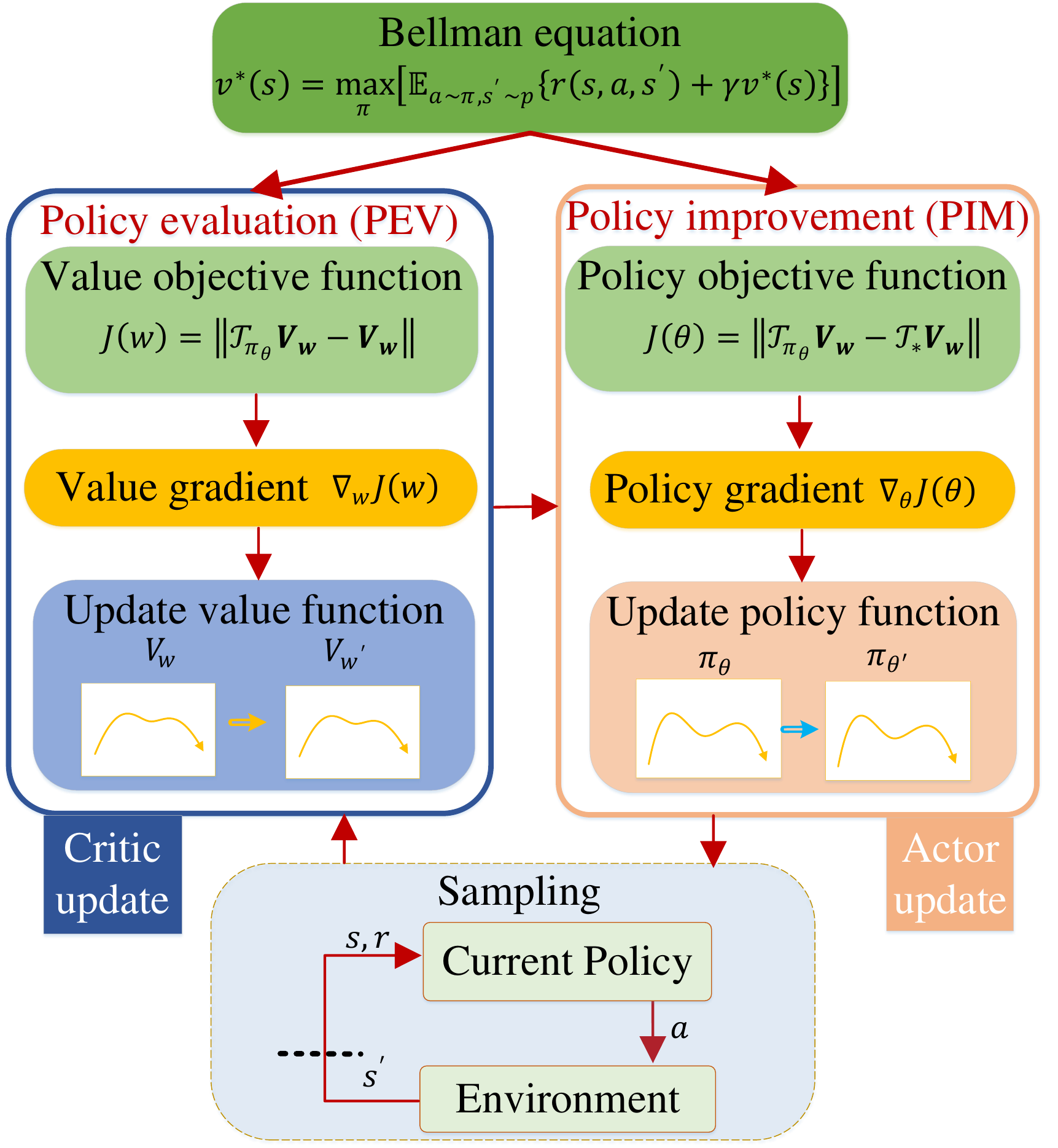}}
\caption{Actor-critic architecture derived from indirect RL.}
\label{indirect_to_actor_critic}
\end{center}
\end{figure}

\begin{algorithm}[tb]
   \caption{Indirect reinforcement learning}
   \label{alg:Approximate policy iteration}
\begin{algorithmic}
   \STATE {\bfseries Initialize:} $\theta$, $w$
   \REPEAT
   \STATE {\bfseries Step 1: Policy Evaluation}
   
   \setlength{\leftskip}{1em}
   \REPEAT
   
   \setlength{\leftskip}{1em}
   \STATE $w \gets \textbf{ValueOptimizer}(\nabla_{w}J(w), w, \theta)$
   \UNTIL {$w$ converges}
   
   \setlength{\leftskip}{0em}
   \STATE {\bfseries Step 2: Policy Improvement}
   
   \setlength{\leftskip}{1em}
   \REPEAT
   
   \setlength{\leftskip}{1em}
   \STATE $\theta \gets \textbf{PolicyOptimizer}(\nabla_{\theta}J(\theta), \theta, w)$
   \UNTIL {$\theta$ converges}
   
   \setlength{\leftskip}{0em}
   \UNTIL{$\theta, w$ converge}
\end{algorithmic}
\end{algorithm}

\subsection{Equivalence}
As we have introduced before, there are several works that have drawn connections between the policy gradient and the value-based methods in the framework of maximum entropy. While the policy gradient method is a typical direct method, the value-based method is actually a special class of indirect RL with an implicit greedy policy, which will be argued in the next section and be verified experimentally in section \ref{sec.experiment}. In this section, similarly, we establish the equivalence between the direct and indirect RL by studying the form of their respective value and policy gradients in the actor-critic architecture, so as to bridge the two families of RL algorithms. Because they have the same value gradient in the critic update, i.e., \eqref{eq.critic_gradient}, we focus on the differences in their actor updates, i.e., comparing the direct PG \eqref{eq.direct_pg} and the indirect PG \eqref{eq.indirect_PG}. Then, we are able to capture two main differences listed below:

\textbf{The state distribution in the policy gradient is different}: Although the direct and indirect policy gradient both seek to optimize the value function under the initial state distribution $d^0$, they have gradient with respect to different state distributions. For direct RL, the value function in its objective is not only a function of state, but also a function of policy parameters $\theta$. When taking gradient of it, we have to unroll it ceaseless until we get the form of \eqref{eq.direct_pg}. Due to the unroll effect, the direct policy gradient becomes an expectation under the DVF $d^{\gamma}$. On the other hand, the indirect policy gradient is an expectation under the initial state distribution.

\textbf{The value function in the policy gradient is different}: the difference is mainly caused by the unalike actor objective functions. The objective function of the direct RL is the true value function, which is a function of policy parameters. Its gradient retains itself on account of the unroll effect. In contrast, the objective function of indirect RL depends on the optimized value function in the PEV step. Then in the PIM step, it is no longer a function of policy and remains unchanged, leading to the approximate value function in the indirect PG.

Although there exist two differences, they can be unified into a policy gradient with the approximate value function and the stationary state distribution. For the direct policy gradient, we replace the true value function with the approximate value function, then choose the initial state distribution as the stationary state distribution or let $\gamma \rightarrow 1$ to turn the DVF to the stationary state distribution according to Proposition \ref{myproposition1} and \ref{myproposition2}. For the indirect policy gradient, we just choose the initial state distribution as the stationary state distribution. The equivalence of direct and indirect RL is shown in Figure \ref{equivalence}.

\begin{figure}[ht]
\begin{center}
\centerline{\includegraphics[width=6cm]{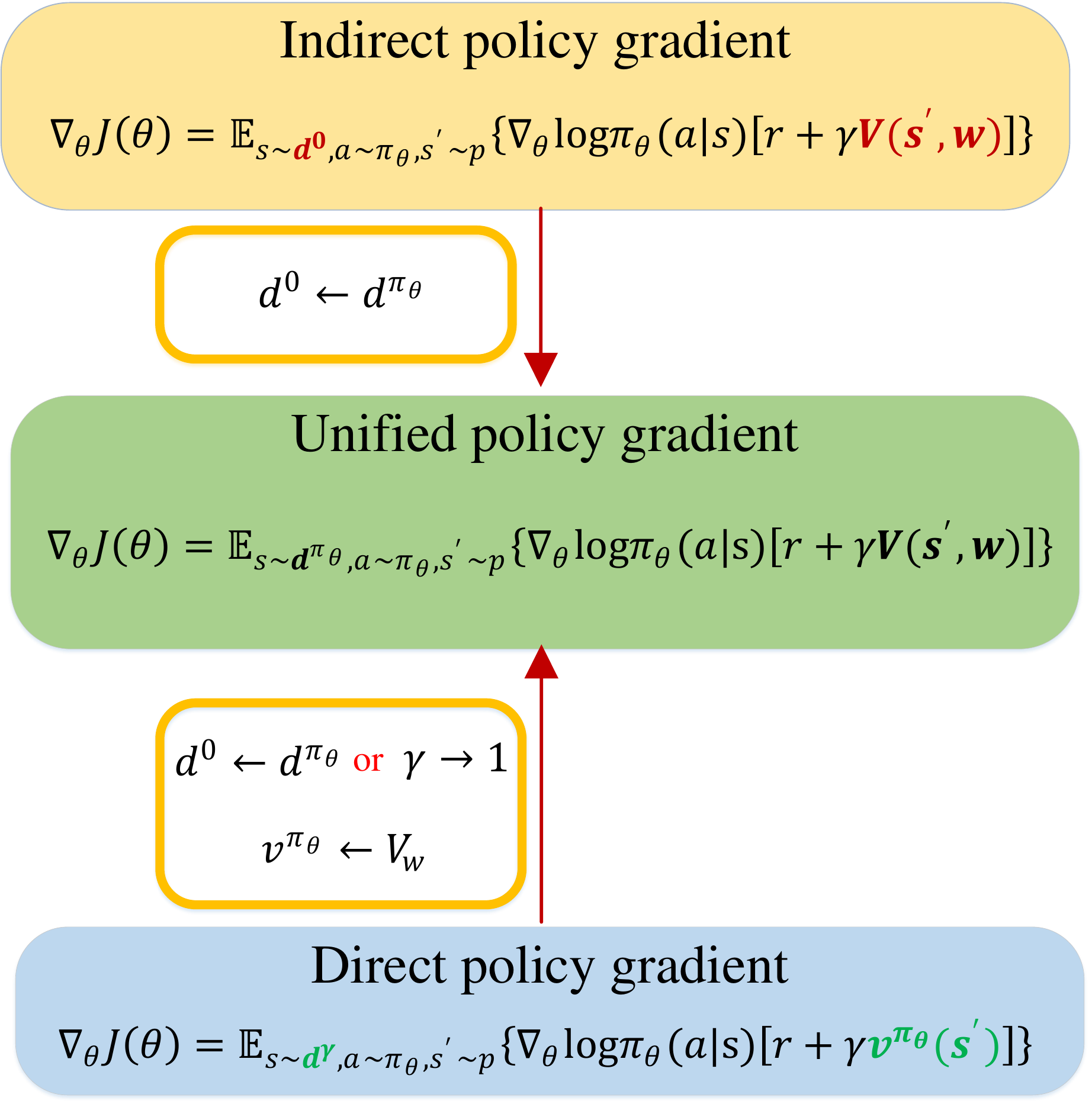}}
\caption{Equivalence of the direct and indirect policy gradient.}
\label{equivalence}
\end{center}
\end{figure}

\subsection{Algorithm classification}
In this section, we classify the recent mainstream RL algorithms by the proposed direct and indirect taxonomy against other taxonomies, value-based and policy-based, model-free and model-based. We first talk about the relationship between our taxonomy and the other two, and finally group the mentioned RL algorithms under these criteria.

For the relationship between our taxonomy and the value-based and policy-based taxonomy, we find that all the value-based methods belong to a special case of indirect method that have an implicit greedy tabular policy. For instance, the Q-learning and its ``deep" variants are in fact implementing one step PEV and one step PIM until converging to the solution of Bellman equation, where the PIM is to find a greedy policy secretly by maximizing the Q-value of the next state. The final solution is the only fixed point of the Bellman operator. While the Bellman operator is formulated by the one-step TD, the C51 and Retrace($\lambda$) try to construct an alternative contraction operator in the distributional and return-based domain, which also guarantees the (distributional) value function converges to the solution of the (distributional) Bellman equation. Likewise, these operators contain one-step PEV and one-step PIM as well, where the PIM step is designed to find a (near) greedy policy. On the other hand, the policy-based method can be either direct method or indirect method, depending on how it has been developed. To name a few direct policy-based methods, Vanilla PG and A3C are proposed with the support of policy gradient theorem; TRPO seeks to optimize a surrogate constrained problem whose objective is a lower bound of the original one \eqref{eq.direct_objective_function}; PPO further simplifies the TRPO problem to a clipped surrogate objective; DPG and its ``deep" variants such as DDPG, TD3, MVE are all put forward on the basis of deterministic policy gradient theorem, which give an analytical PG in deterministic domain. The common feature or motivation of these algorithms is that they intend to directly improve a certain objective about a long-term goal rather than to solve the Bellman optimality condition as does in the indirect policy-based methods. For example, the soft Q-learning constructs a fixed point iteration in the framework of entropy regularization to approach the solution of the entropy augmented Bellman equation. Different from the value-based methods, it owns a paramaterized policy. And different from the direct methods, the policy is optimized by minimizing the distance between the Boltzmann policy in each step, which is the optimal solution in each PIM step suggested by the framework. Similar idea can be found in SAC and DSAC. Still in the entropy framework, PCL finds an alternative necessary condition of the optimal solution. Based on that the update loss is established, realizing joint updates of the value and policy. Moreover, the whole class of ADP algorithms is developed and analyzed in the scope of solving Bellman equation, thus belongs to the indirect method. On the other hand, there is basically no straight relationship with the model-based and model-free taxonomy and ours since they are fundamentally two orthogonal classification criteria from distinctive point of view. The algorithm categorizations are shown in Table \ref{tab.classification1} and Table \ref{tab.classification2}. Most of these algorithms have been introduced in section \ref{sec.intro}.
\begin{table*}[htbp]
\captionsetup{justification=centering,labelsep=newline,font=small}
\caption{Comparison with the value-based and policy-based taxonomy.}
\label{tab.classification1}
\begin{center}
\begin{small}
\begin{sc}
\begin{tabular}{p{1cm}<{\centering}p{5cm}<{\centering}p{5cm}<{\centering}}
\hline
\headrow
 & Value-based & Policy-based \\
\hline
Indirect & DP \cite{bellman1966dynamic}, Q-learning \cite{watkins1992q} & Soft Q-learning \cite{haarnoja2017reinforcement}, HDP \cite{white1992handbook}, DHP \cite{white1992handbook}, \\
& DQN \cite{mnih2015human}, DDQN \cite{van2016deep} &  ADHDP \cite{white1992handbook}, ADDHP \cite{white1992handbook} \\
& Dueling DQN \cite{wang2015dueling}, PER \cite{schaul2015prioritized}  & GDHP \cite{prokhorov1997adaptive}, ADGDHP \cite{prokhorov1997adaptive}\\
& C51 \cite{bellemare2017distributional}, Rainbow \cite{hessel2018rainbow}  & CDADP \cite{duan2019deep}, DGPI \cite{duan2019generalized}\\ 
& NAF \cite{gu2016continuous}, VPN \cite{oh2017value}, GATS \cite{azizzadenesheli2018sample} & SVG \cite{heess2015learning},  SAC \cite{haarnoja2018soft},  DSAC \cite{duan2020addressing} \\
& Retrace($\lambda$) \cite{munos2016safe}, Dyna-DQN \cite{holland2018effect}&  PCL \cite{nachum2017bridging}, Trust-PCL \cite{nachum2017trust}\\

\hline
Direct      & & Natural PG \cite{kakade2002natural}, TRPO \cite{schulman2015trust}, PPO \cite{schulman2017proximal},\\
            & & DPG \cite{silver2014deterministic}, Off-PAC \cite{degris2012off}, ACER \cite{wang2016sample}, \\
            & & Reactor \cite{gruslys2017reactor}, IPG \cite{gu2017interpolated},  ACE \cite{imani2018off},\\
            & & Geoff-AC \cite{zhang2019generalized}, DDPG \cite{lillicrap2015continuous}, TD3 \cite{fujimoto2018addressing}, \\
            & & ACKTR \cite{wu2017scalable}, SIL \cite{oh2018self}, A3C (A2C) \cite{mnih2016asynchronous},\\
            & & APE-X \cite{horgan2018distributed}, IMPALA \cite{espeholt2018impala}, I2A \cite{racaniere2017imagination},  \\
            & & MVE \cite{feinberg2018model}, STEVE \cite{buckman2018sample}, GPS \cite{levine2013guided}, \\
            & & MBPO \cite{janner2019trust}, PILCO \cite{deisenroth2011pilco}, ME-TRPO \cite{kurutach2018model}\\
            & & ES \cite{salimans2017evolution}, ARS \cite{mania2018simple} D4PG \cite{d4pg2018}\\
            & & Recurrent world models \cite{ha2018recurrent}, Dreamer \cite{DREAMER}\\


\hline
\end{tabular}
\end{sc}
\end{small}
\end{center}
\vskip -0.1in
\end{table*}

\begin{table*}[htbp]
\captionsetup{justification=centering,labelsep=newline,font=small}
\caption{Comparison with the model-based and model-free taxonomy.}
\label{tab.classification2}
\begin{center}
\begin{small}
\begin{sc}
\begin{tabular}{p{1cm}<{\centering}p{5cm}<{\centering}p{5cm}<{\centering}}
\hline
\headrow
 & Model-based & Model-free \\
\hline
Indirect    & DP \cite{bellman1966dynamic}, HDP \cite{white1992handbook}, & Retrace($\lambda$) \cite{munos2016safe}, Q-learning \cite{watkins1992q},\\
            & DHP \cite{white1992handbook}, ADHDP \cite{white1992handbook}, & DQN \cite{mnih2015human}, DDQN \cite{van2016deep}, \\ 
            & ADDHP \cite{white1992handbook}, GDHP \cite{prokhorov1997adaptive},         & PER \cite{schaul2015prioritized}, Dueling DQN \cite{wang2015dueling},\\
            & ADGDHP \cite{prokhorov1997adaptive},CDADP \cite{duan2019deep},        & C51 \cite{bellemare2017distributional}, Rainbow \cite{hessel2018rainbow},\\
            & DGPI \cite{duan2019generalized}, NAF \cite{gu2016continuous}, Dyna-DQN \cite{holland2018effect}      & PCL \cite{nachum2017bridging}, Trust-PCL \cite{nachum2017trust}\\
            & SVG \cite{heess2015learning}, VPN \cite{oh2017value}, GATS \cite{azizzadenesheli2018sample}          & SAC \cite{haarnoja2018soft}, DSAC \cite{duan2020addressing}\\

\hline
Direct         & MVE \cite{feinberg2018model} & Natural PG \cite{kakade2002natural}, TRPO \cite{schulman2015trust} \\
               & STEVE \cite{buckman2018sample}  &  ACKTR \cite{wu2017scalable}, PPO \cite{schulman2017proximal} \\
               & ME-TRPO \cite{kurutach2018model}& DPG \cite{silver2014deterministic}, Off-PAC \cite{degris2012off} \\
               & PILCO \cite{deisenroth2011pilco}& ACER \cite{wang2016sample}, Reactor \cite{gruslys2017reactor}  \\
               & MBPO \cite{janner2019trust}         & IPG \cite{gu2017interpolated}, TD3 \cite{fujimoto2018addressing}, ACE \cite{imani2018off} \\
               & GPS \cite{levine2013guided}     & Geoff-AC \cite{zhang2019generalized}, DDPG \cite{lillicrap2015continuous} \\
               & I2A \cite{racaniere2017imagination} & Soft Q-learning \cite{haarnoja2017reinforcement}, SIL \cite{oh2018self} \\
               & Recurrent world models \cite{ha2018recurrent}     & A3C (A2C) \cite{mnih2016asynchronous}, APE-X \cite{horgan2018distributed} \\
               & Dreamer \cite{DREAMER}          & IMPALA \cite{espeholt2018impala}, ES \cite{salimans2017evolution}, ARS \cite{mania2018simple}, D4PG \cite{d4pg2018}\\

\hline
\end{tabular}
\end{sc}
\end{small}
\end{center}
\vskip -0.1in
\end{table*}

\section{Convergence results}\label{sec.convergence}
In this section, we establish the convergence results of the direct RL and indirect RL with the help of the stochastic gradient theorem \cite{bertsekas2000gradient} and the error bound for approximate policy iteration \cite{bertsekas1996neuro}. With minor modifications to fit our settings, we can apply their theoretical results immediately.
\subsection{Direct RL}
Before getting into the convergence analysis of the direct method, we first define some symbols for convenience. Consider the following process,
\begin{equation}\label{original_notation_process}
    \theta_{k+1}=\theta_k+\alpha_k \overline{\nabla_{\theta} J(\theta)}|_{\theta=\theta_k},
\end{equation}
where $\overline{\nabla_{\theta}J(\theta)}=\frac{1}{T}\sum_{(s,a,r,s')\in\mathcal{D}}\nabla_{\theta} \log \pi_{\theta}(a | s)\left[r\!+\!\gamma V_w(s')\right]$ is the numerical approximation of the direct policy gradient \eqref{eq.direct_pg} and $\mathcal{D}=\left\{(s_i,a_i,r_i,s_{i+1})\right\}_{i=0:T-1}$ is a batch of data generated by the policy $\pi_{\theta}$. We denote that $\nabla J(\theta_k)=\nabla_{\theta} J(\theta)|_{\theta=\theta_k}$, $\overline{\nabla J(\theta_k)}=\overline{\nabla_{\theta} J(\theta)}|_{\theta=\theta_k}$, $w_k=\overline{\nabla J(\theta_k)}-\nabla J(\theta_k)$ and $\mathcal{F}_k=\left\{\theta_0,\alpha_0,\nabla J(\theta_0),w_0,...,\theta_k,\alpha_k,\nabla J(\theta_k)\right\}$, then the process \eqref{original_notation_process} becomes
\begin{equation}
\nonumber
    \theta_{k+1}=\theta_k+\alpha_k (\nabla J(\theta_k)+w_k).
\end{equation}
With the above formulation, the convergence of the direct method can be indicated by the following theorem.

\begin{theorem}(Stochastic gradient theorem \cite{bertsekas2000gradient})
Let $\theta_k$ be a sequence generated by the method
\begin{equation*}
    \theta_{k+1}=\theta_k+\alpha_k (\nabla J(\theta_k)+w_k),
\end{equation*}
where $\alpha_k$ is a deterministic positive stepsize, $\nabla J(\theta_k)$ is the steepest ascent direction, and $w_k$ is a random noise term. Let $\mathcal{F}_k$ be an increasing sequence of $\sigma$-fields. We assume the following conditions:

(a) $\nabla J(\theta_k)$ and $\theta_k$ are $\mathcal{F}_k$-measurable.

(b) (Lipschitz continuity of $\nabla J$) The function $J$ is continuously differentiable and there exists a constant $L$ such that
\begin{equation*}
    \left\| \nabla J(\theta_1) - \nabla J(\theta_2) \right\| \leq L \left\| \theta_1 - \theta_2 \right\|, \forall \theta_1, \theta_2 \in \mathbb{R}^m.
\end{equation*}

(c) We have, for all $t$ and with probability 1,
\begin{equation*}
\mathbb{E}\left\{w_{k} | \mathcal{F}_{k}\right\}=0,
\end{equation*}
and
\begin{equation*}
\mathbb{E}\left\{\left\|w_{k}\right\|^{2} | \mathcal{F}_{k}\right\} \leq A\left(1+\left\|\nabla J\left(\theta_{k}\right)\right\|^{2}\right),
\end{equation*}
where $A$ is a positive deterministic constant.

(d) The stepsize $\alpha_k$ is positive and satisfies
\begin{equation*}
\sum_{t=0}^{\infty} \alpha_{k}=\infty, \quad \sum_{t=0}^{\infty} \alpha_{k}^{2}<\infty.
\end{equation*}
Then either $J(\theta_k)\rightarrow \infty$ or $J(\theta_k)$ converges to a finite value and $\lim _{k \rightarrow \infty} \nabla J\left(\theta_{k}\right)=0$, $a.e.$. Furthermore, every limit point of $\theta_k$ is a stationary point of $J$.
\end{theorem}

\subsection{Indirect RL}
The indirect method generates a sequence of policies $\pi_{\theta_k}$ and a corresponding sequence of approximate value functions $\bm{V_{w_k}}$. We use the following assumptions to bound the error of the PEV and PIM steps:
\begin{equation}\label{eq.indirect_convergence_pev_error_bound}
\lVert \bm{V_{w_k}}-\bm{v^{\pi_{\theta_k}}}\rVert_{\infty}\leq \epsilon, \quad k=0,1, \ldots
\end{equation}
\begin{equation}\label{eq.indirect_convergence_pim_error_bound}
    \lVert \mathcal{T}_{\pi_{\theta_{k+1}}}\bm{V_{w_k}} - \mathcal{T}_* \bm{V_{w_k}}\rVert_{\infty}\leq \delta, \quad k=0,1, \ldots
\end{equation}
where $\lVert \bm{x} \rVert_{\infty}=\max_{i=1:n} |x_i|$ is the infinite norm, $\epsilon$ and $\delta$ are some positive scalars. The scalar $\epsilon$ is an assumed worst-case bound on the error incurred during the PEV step. The scalar $\delta$ is a bound on the error incurred in the course of the computations required for the PIM step. Then the following theorem reveals the convergence of the indirect method.
\begin{theorem}(Error bound for approximate policy iteration \cite{bertsekas1996neuro})
Given \eqref{eq.indirect_convergence_pev_error_bound} and \eqref{eq.indirect_convergence_pim_error_bound}, the sequence of policies $\pi_{\theta_k}$ generated by the indirect method satisfy

\begin{equation*}
    \limsup _{k \rightarrow \infty} \lVert \bm{v^{\pi_{\theta_{k}}}}-\bm{v^{*}}\rVert_{\infty} \leq \frac{\delta+2 \gamma \epsilon}{(1-\gamma)^{2}}.
\end{equation*}
\end{theorem}

\section{Experimental verification}\label{sec.experiment}

In this section, we carry out two experiments to answer the following two questions.

1. Is the value-based method (one without an explicit policy) a special case of the indirect method (one with an explicit policy)?

2. What is the difference in the training process of the direct PG, indirect PG and the unified PG? In other words, what is the influence of the two PG components (state distribution and value function) on the training process?

We discuss the results of running different methods on a toy 4 by 6 grid world, as shown in Figure \ref{fig.task}. The agent always begins in the square colored blue and the episode continues until it reaches the terminal square colored red, upon which it receives a reward of 1. All other times it receives no reward. The discount factor is chosen as $\gamma=0.9$. The squares colored black and the border of the arena are regarded as the walls. The white squares are the free-to-go space, resulting in 16 discrete states in this task numbered as in Figure \ref{fig.task.state_number}. In each state the agent can perform one of the four actions including ``up", ``down", ``left" and ``right" and transfer to the adjacent state or stay still when it encounters the wall.

\begin{figure*}[h]
\centering
\captionsetup[subfigure]{justification=centering}
\subfloat[]{\includegraphics[width=0.3\linewidth]{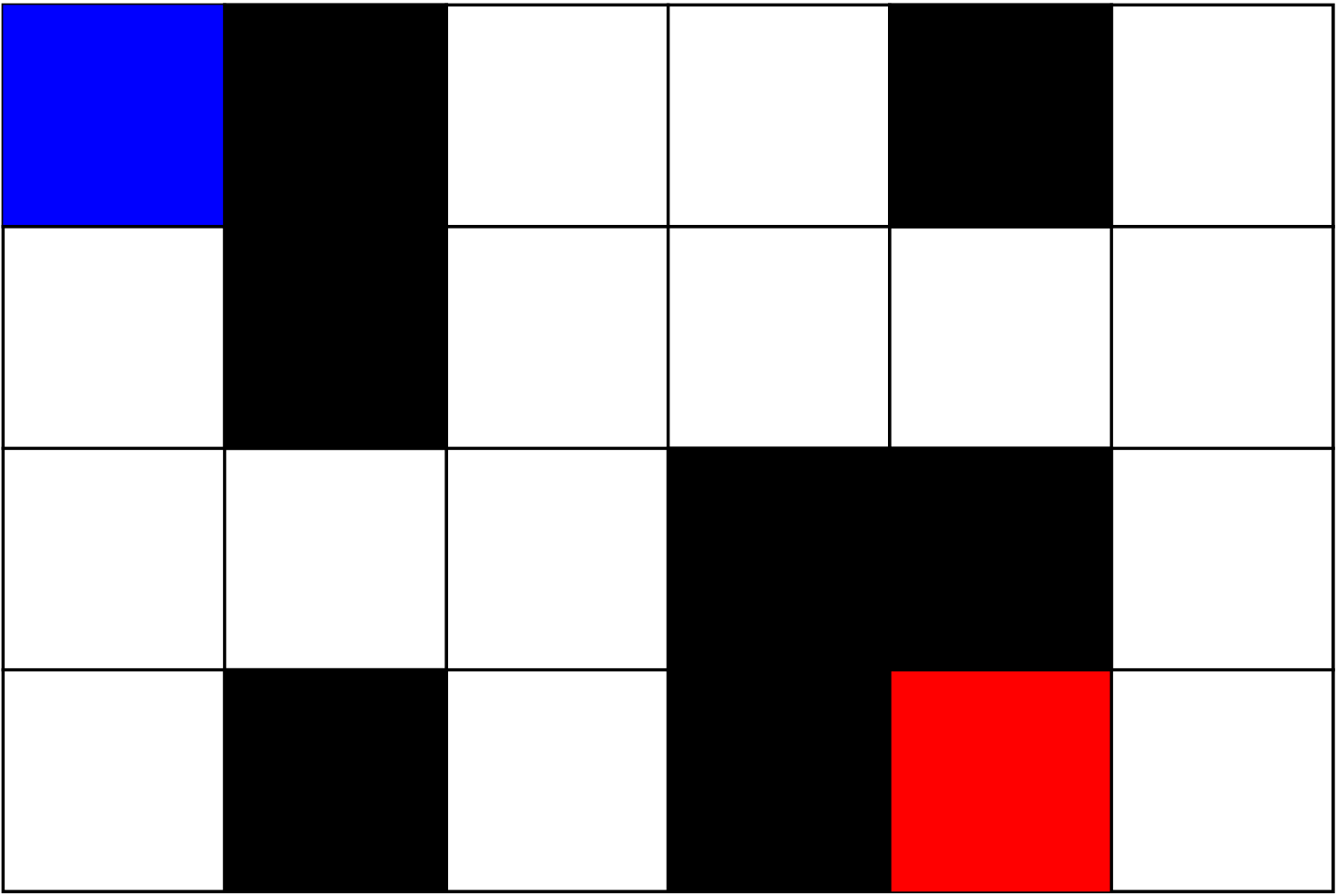}}\qquad
\subfloat[]{\label{fig.task.state_number}\includegraphics[width=0.3\linewidth]{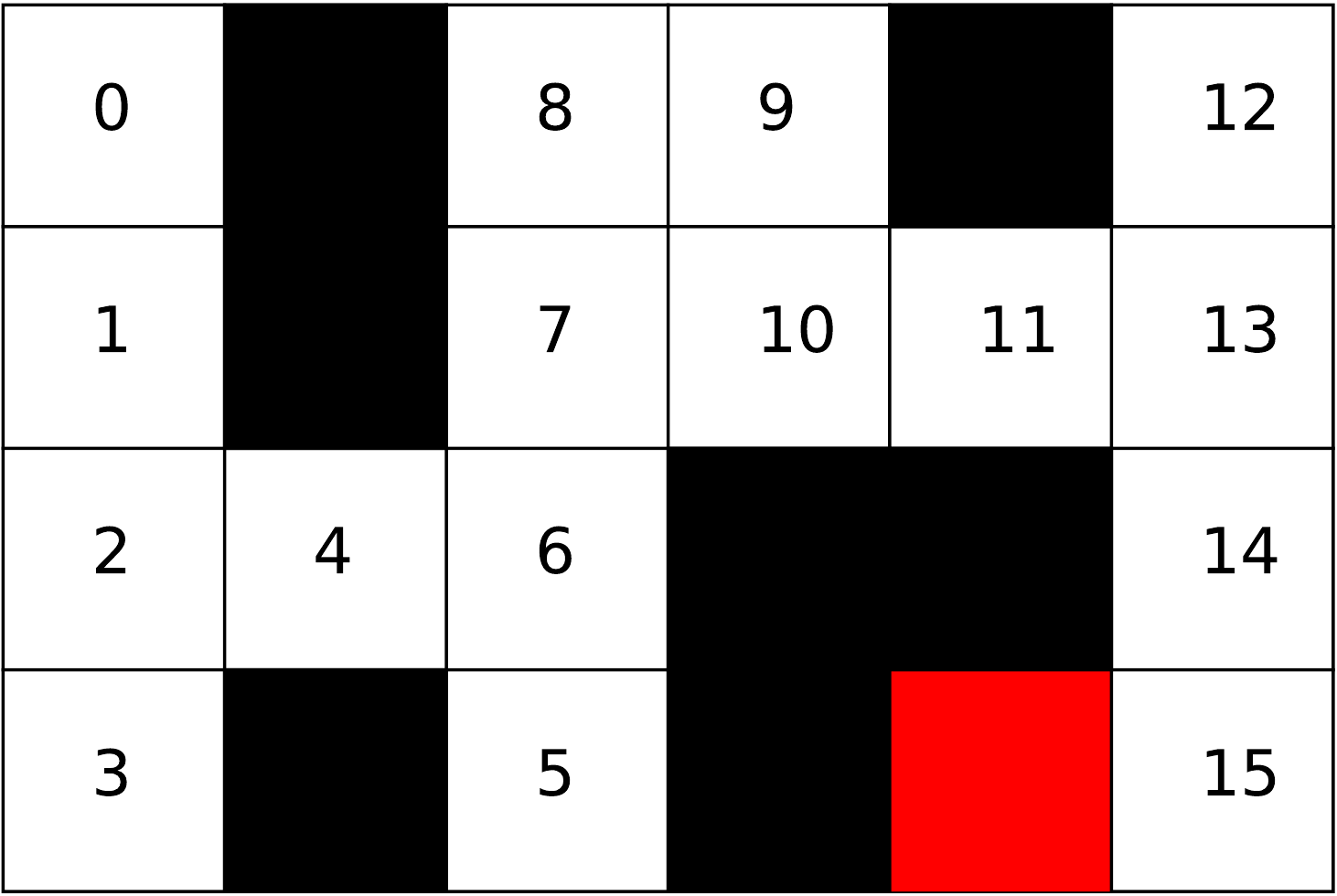}}
\caption{The grid world task used in the experiment and the state number.}
\label{fig.task}
\end{figure*}

\subsection{Experiment 1}
\begin{figure*}[h]
\begin{center}
\centerline{\includegraphics[width=0.88\linewidth]{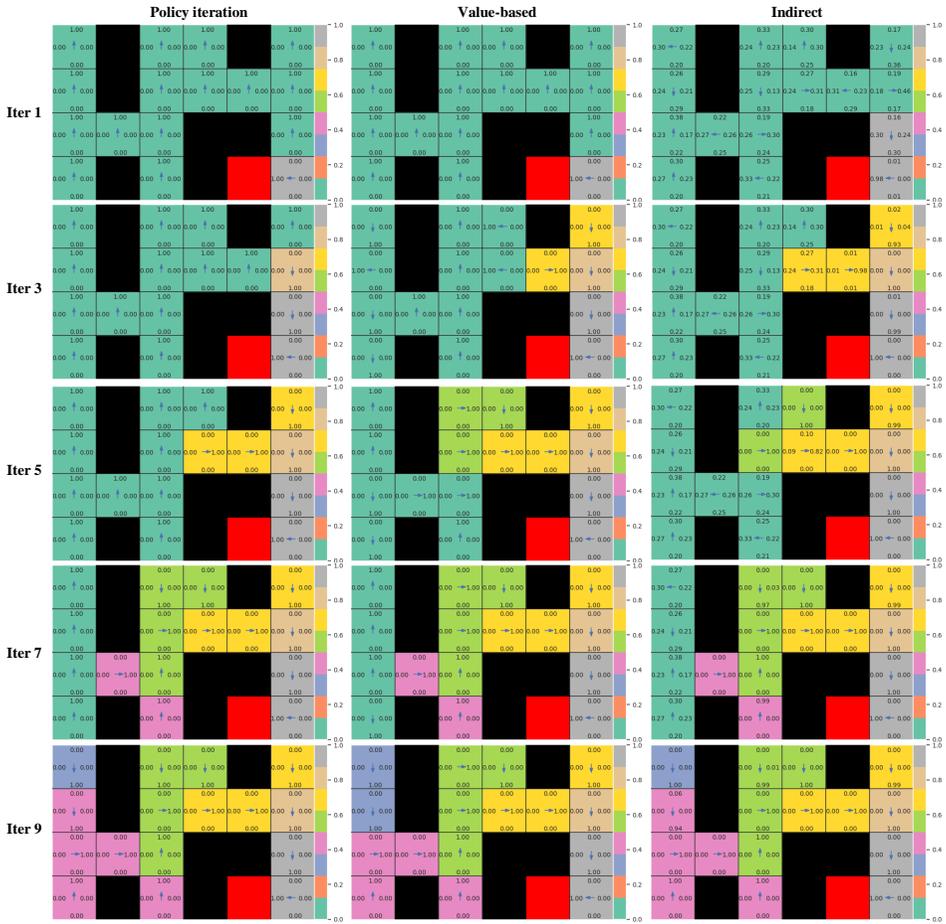}}
\caption{Visualization of the values and policys during the training process across different methods.}
\label{fig.exp1}
\end{center}
\end{figure*}
In this experiment, we compare a value-based method with an approximate value function and an indirect method with both approximate policy and value functions. Besides, a policy iteration algorithm with tabular value and policy is also performed as the benchmark. For each method, we initialize the value function with the same value (all zeros) and initialize a dummy policy which chooses to move ``up" in every state. Then, we perform the PIM and PEV procedures alternatively, each of which does not terminate until the policy or the value hits convergence. 
For the methods using value function approximation, i.e. value-based and indirect methods, the function is a multi-layer perceptron (MLP) with three hidden layers, each with 256 units and GELU activation \cite{hendrycks2016gaussian}. The input is an one-hot encoding of the 16 states and the output is the corresponding value.
It is also worth noting that the policy iteration and the value-based methods aim to find an implicit tabular greedy policy from scratch in each PIM while the indirect method keeps a parametrerized stochastic policy from the very start and seeks to optimize it by the policy gradient in each iteration. The policy is parameterized by a 16$\times$4 matrix normalized by the Softmax layer to denote the discrete distribution of the four actions in each of the 16 states.
We visualize the values and policies during the training process, shown in Figure \ref{fig.exp1}. In each iteration, each state, i.e. a square, is covered by a certain color to denote its value. And the number on each side of it represents the probability moving towards that direction under the current policy. The arrow at the square center points to the action with the highest probability.

We can see from the result that both the value-based method and the indirect method are basically in accordance with the benchmark (policy iteration method) in terms of the state value, except some minor differences caused by the function approximation error. In addition, although the indirect method owns a stochastic policy and optimizes it by policy gradient, it yields an action distribution close to the greedy policy produced by the value-based method. Furthermore, to eliminate the randomness, we further perform 5 runs of value-based and indirect method respectively and calculate the value mean squared error (MSE) in each iteration, as shown in Table \ref{tab.exp1}. Both the errors are close to zero, indicating their equivalence in statistics. The experiment results support the point that the value-based method is in fact a special indirect method with an implicit greedy policy.

\begin{table*}[t]
\captionsetup{justification=centering,labelsep=newline,font=small}
\caption{Mean squared error of the state value over 5 runs, $\pm$ corresponds to 95\% confidence interval over runs.}
\label{tab.exp1}
\begin{center}
\begin{small}
\begin{sc}
\begin{tabular}{p{1.5cm}<{\centering}p{1.1cm}<{\centering}p{1.1cm}<{\centering}p{1.1cm}<{\centering}p{1.1cm}<{\centering}p{1.1cm}<{\centering}p{1.1cm}<{\centering}p{1.1cm}<{\centering}p{1.1cm}<{\centering}}
\hline
\headrow
Iteration & 1 & 3 & 5 & 7 & 9 & 11& 13 & 15\\
\hline
Value-based method & 5e-7 ($\pm$2e-6) & 8e-4 ($\pm$2e-3) & 2e-3 ($\pm$8e-3) & 8e-4 ($\pm$1e-5) & 7e-4 ($\pm$1e-3) & 1e-7 ($\pm$1e-7) & 1e-7 ($\pm$1e-7) & 1e-7 ($\pm$1e-7)\\
\hline
Indirect method & 2e-8 ($\pm$4e-8) & 5e-3 ($\pm$1e-2) & 1e-2 ($\pm$3e-2) & 3e-3 ($\pm$4e-3) & 7e-4 ($\pm$1e-3) & 9e-8 ($\pm$1e-7) & 8e-8 ($\pm$1e-7) & 8e-8 ($\pm$1e-7) \\
\hline
\end{tabular}
\end{sc}
\end{small}
\end{center}
\end{table*}

\begin{table*}[htbp]
\captionsetup{justification=centering,labelsep=newline,font=small}
\caption{Basic description of the methods for comparison.}
\label{tab.exp2.method}
\begin{center}
\begin{small}
\begin{sc}
\begin{tabular}{p{3cm}<{\centering}p{4cm}<{\centering}p{4cm}<{\centering}}
\hline
\headrow
Method & State distribution in PG & Value function in PG\\
\hline
Direct PG & DVF & True \\
\hline
Indirect PG & Initial & Approximation \\
\hline
Unified PG & Stationary & Approximation \\
\hline
Baseline PG 1 & Initial & True \\
\hline
Baseline PG 2 & Stationary & True\\
\hline
\end{tabular}
\end{sc}
\end{small}
\end{center}
\end{table*}

\subsection{Experiment 2}
In this experiment, we compare the direct PG, the indirect PG, the unified PG and two other baseline PGs to study the influence of different PG components, i.e., state distribution and value function, on the training process. The direct method uses the PG with the discounted visiting frequency (DVF) and the ``true" value function, while the indirect method uses the PG with the initial state distribution and the approximate value function. The unified PG is with the stationary state distribution and the approximate value function. The two baseline PGs all use true value function in the PG but employ initial state distribution and stationary state distribution respectively. All the methods are listed in Table \ref{tab.exp2.method}.
The DVF are sampled by collecting different number of states in different steps starting from the initial state distribution with exponential decay constant $\gamma$ as suggested by the definition of the DVF \eqref{eq.DVF}. Similarly, guided by the Proposition \ref{myproposition2}, the stationary state distribution is therefore obtained by using the DVF with $\gamma=1$. The ``true" value function is estimated by Monte Carlo method, while the approximate value function is constructed and learned the same way as in the experiment 1 except that it is updated a fixed number of time $m$ between two policy updates and does not necessarily converge. In addition, we use the same policy as that in the experiment 1. The value and policy networks are initialized by the same seed across different methods.
We design two kinds of initial state distribution $d^0$ and two different value iteration numbers $m$, which formulates 4 cases in total shown in Table \ref{tab.exp2.cases}. In each case, we implement by all the methods and evaluate their true value mean in terms of the initial state distribution and the stationary state distribution. The training results are presented in the Figure \ref{fig.exp2.value_mean}.
\begin{table*}[htbp]
\captionsetup{justification=centering,labelsep=newline,font=small}
\caption{Different case settings for experiment 2.}
\label{tab.exp2.cases}
\begin{center}
\begin{small}
\begin{sc}
\begin{tabular}{p{3cm}<{\centering}p{4cm}<{\centering}p{4cm}<{\centering}}
\hline
\headrow
Case & Initial state distribution $d^0$ & Value iteration numbers $m$ \\
\hline
1 & $s^0 \sim \mathcal{U}\{0, 15\}$ & 5 \\
\hline
2 & $s^0 \sim \mathcal{U}\{0, 15\}$ & 30 \\
\hline
3 & $s^0 \sim \mathcal{U}\{0, 14\}$ & 5 \\
\hline
4 & $s^0 \sim \mathcal{U}\{0, 14\}$ & 30 \\
\hline
\end{tabular}
\end{sc}
\end{small}
\end{center}
\end{table*}

\begin{figure*}[htbp]
\centering
\captionsetup[subfigure]{justification=centering}
\subfloat[Case 1: $s^0\sim \mathcal{U}\{0, 15\}, m=5$]{\label{fig.exp2.value_mean_case1}\includegraphics[width=0.45\linewidth]{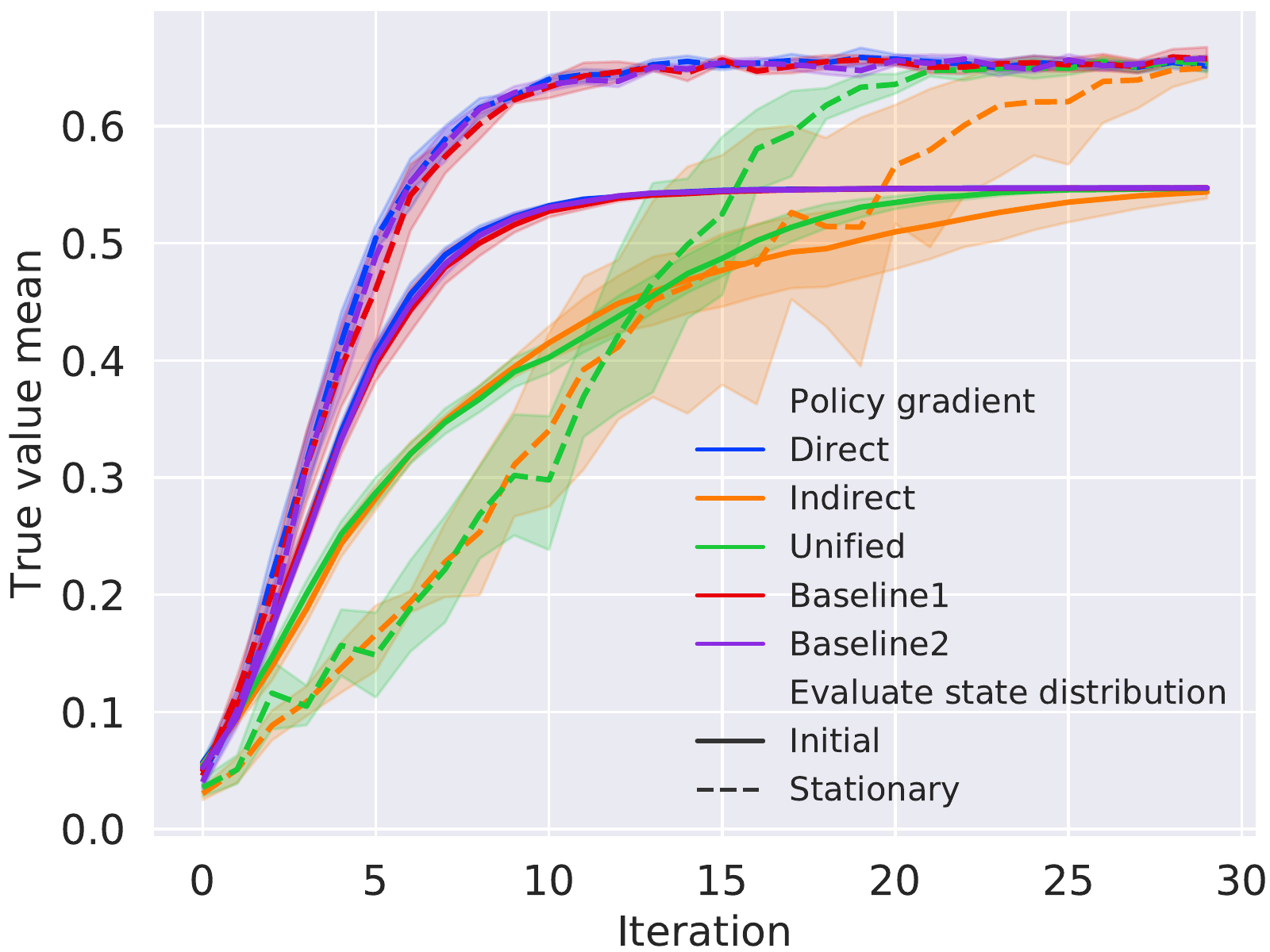}}\qquad
\subfloat[Case 2: $s^0\sim \mathcal{U}\{0, 15\}, m=30$]{\label{fig.exp2.value_mean_case2}\includegraphics[width=0.45\linewidth]{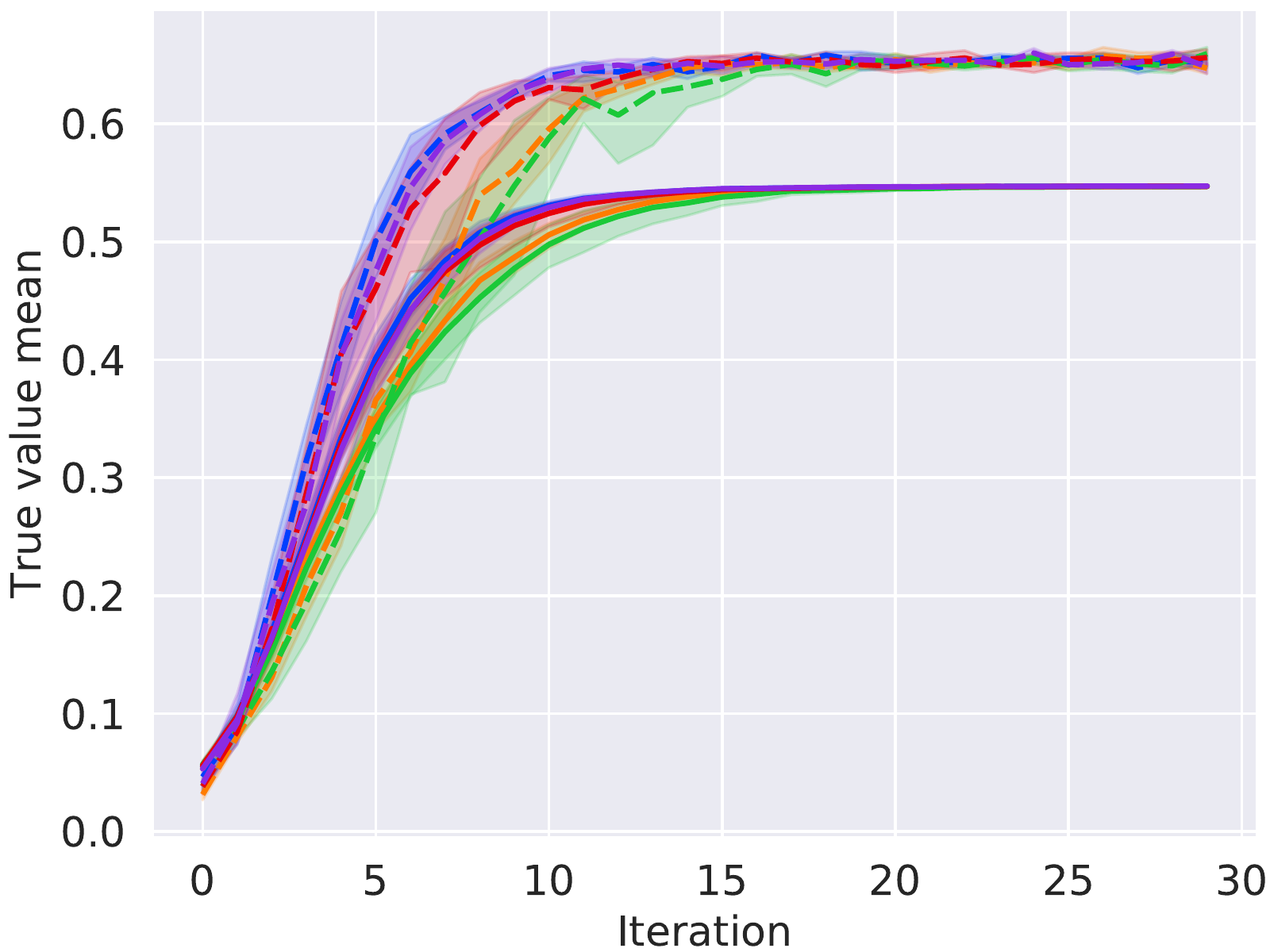}}\\
\subfloat[Case 3: $s^0\sim \mathcal{U}\{0, 14\}, m=5$]{\label{fig.exp2.value_mean_case3}\includegraphics[width=0.45\linewidth]{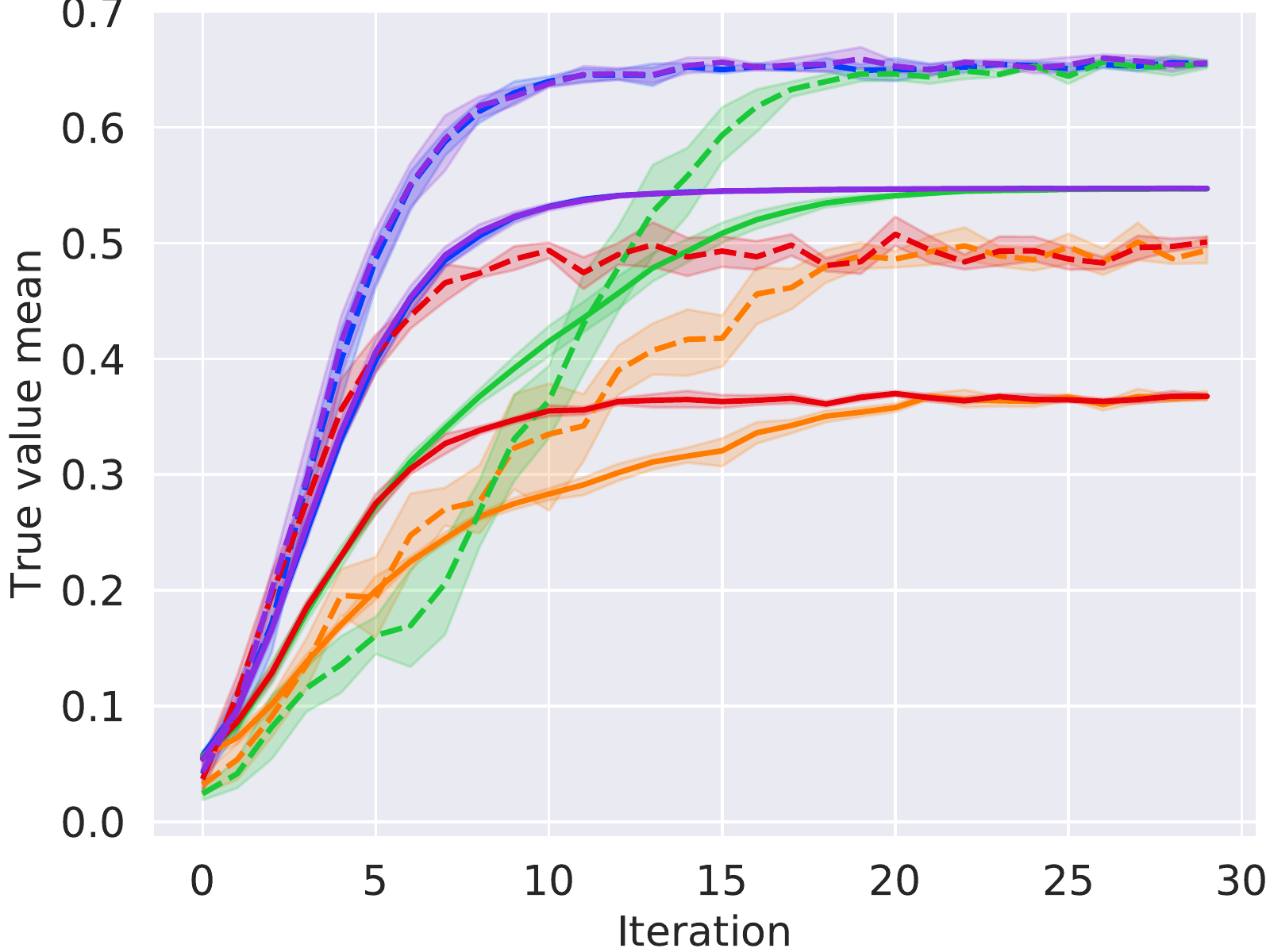}}\qquad
\subfloat[Case 4: $s^0\sim \mathcal{U}\{0, 14\}, m=30$]{\label{fig.exp2.value_mean_case4}\includegraphics[width=0.45\linewidth]{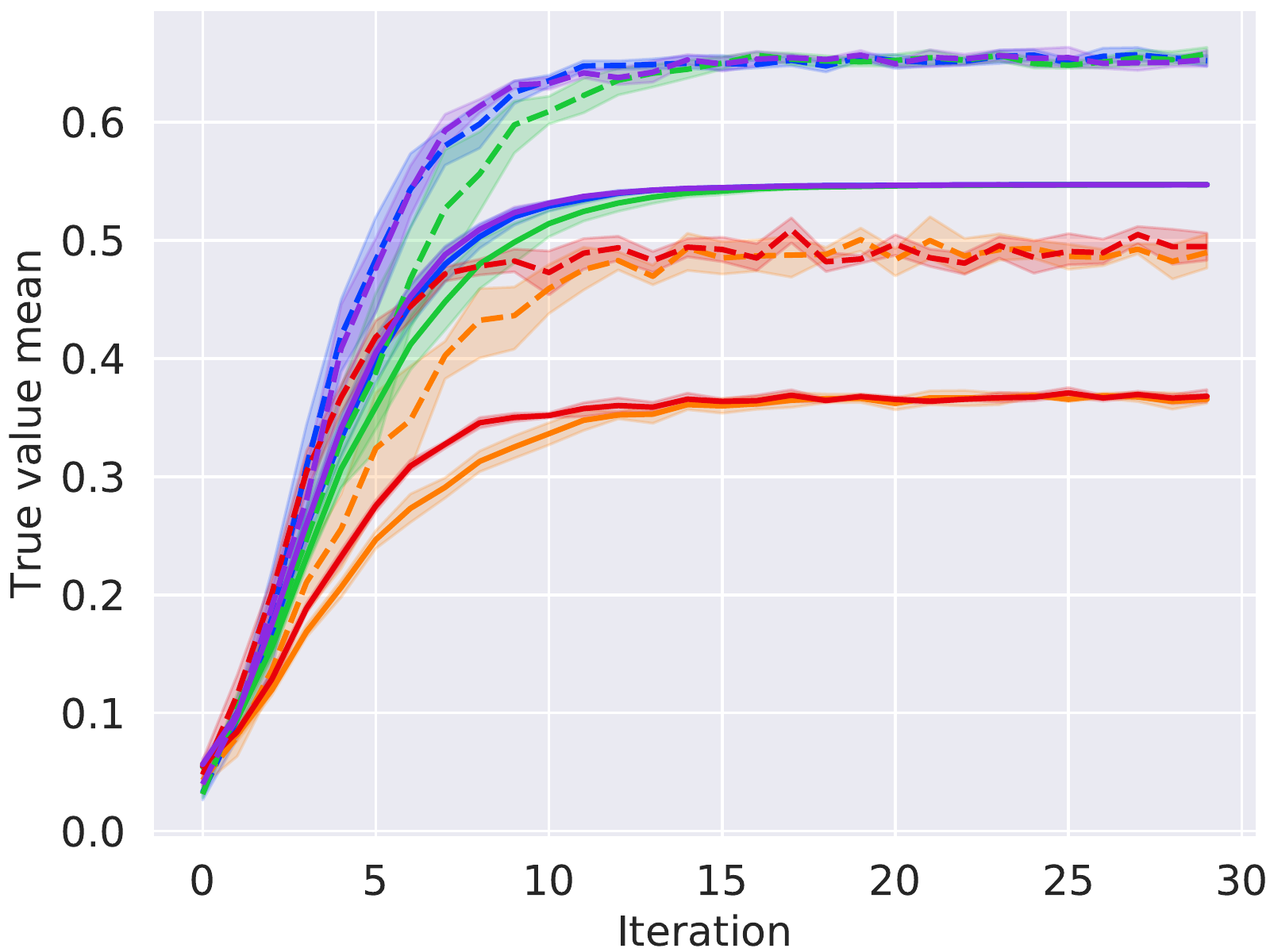}}
\caption{Training process of all methods under different cases. The solid lines correspond to the mean and the shaded regions correspond to 95\% confidence interval over 5 runs.}
\label{fig.exp2.value_mean}
\end{figure*}

We focus on the case 3 (Figure \ref{fig.exp2.value_mean_case3}), where the initial state distribution is discrete uniform distribution $\mathcal{U}\{0, 14\}$, and the value update number between two policy updates is five. An obvious observation is that the value mean evaluated over the stationary state distribution outperforms the one evaluated over the initial state distribution across all methods. That means the learned policy always visits more frequently the states with higher values. To see the effect of state distribution in PG, we compare the direct PG, and two baseline PGs, where they all use the true value function but adopt initial and stationary distributions as the state distribution in their own PG. The figure shows that the direct PG and the baseline PG 2 have similar trends and final performance, because in this problem both the DVF and the stationary state distribution have access to all the states even if the initial state distribution excludes the state ``15". On the other hand, the baseline PG 1 has degraded performance because the policy in state ``15" is not optimized, which further reduces the value of all the associated states. Similar phenomenon can be seen from the comparison of the indirect and the unified PG. However, when the state ``15" is included in the initial distribution, such as in case 1 and case 2 (Figure \ref{fig.exp2.value_mean_case1} and Figure \ref{fig.exp2.value_mean_case2}), the gap between the direct PG and the baseline PG 1 will be eliminated. Moreover, to see the influence of the value function in PG, we compare the indirect PG and the baseline PG 1, where they both use initial state distribution but employ approximation and true value functions separately. Still in the case 3, we can see that the indirect PG converges slower than the baseline1 PG. That is caused by the bias between the approximation value and the true value. In this case, the approximation value function only updates five times between two policy updates, which would not converge until the late training phase and results in inaccurate PG estimations. When we increase the $m$, as shown in case 4 and case 2 (Figure \ref{fig.exp2.value_mean_case4} and Figure \ref{fig.exp2.value_mean_case2}), the margin would be largely reduced. Similar results can be observed by comparing the unified PG and the baseline PG 2.

Except the value curves, we also draw the policy entropy during the training process in the case 1 and case 4, as shown in Figure \ref{fig.exp2.policy_entropy}. The policy entropy keeps decreasing along the training process, that means the policy is gradually optimized and becomes more and more deterministic until it converges to a greedy policy with zero entropy, as the results in the experiment 1. Note that the policy entropy of the indirect and unified PG drops slower, which is again caused by the inaccuracy of the value approximation. Additionally, notice that the indirect and baseline PG 1 converge to a non-zero entropy in the case 4 (Figure \ref{fig.exp2.policy_entropy_case4}). That is originated from the excluded state ``15" because the policy in it is never optimized.
\begin{figure*}[htbp]
\centering
\captionsetup[subfigure]{justification=centering}
\subfloat[Case 1: $s^0\sim \mathcal{U}\{0, 15\}, m=5$]{\label{fig.exp2.policy_entropy_case1}\includegraphics[width=0.45\linewidth]{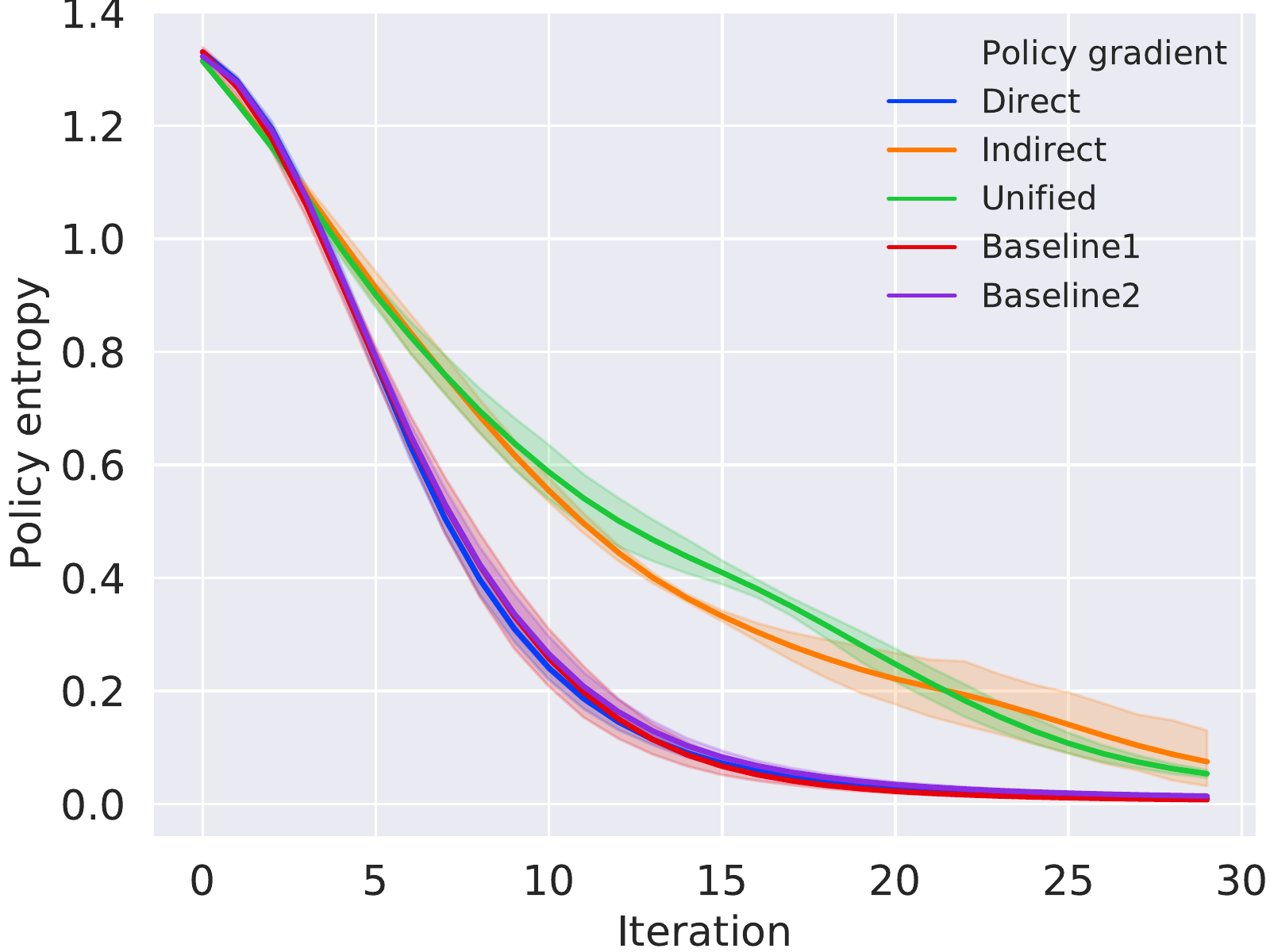}}\qquad
\subfloat[Case 4: $s^0\sim \mathcal{U}\{0, 14\}, m=30$]{\label{fig.exp2.policy_entropy_case4}\includegraphics[width=0.45\linewidth]{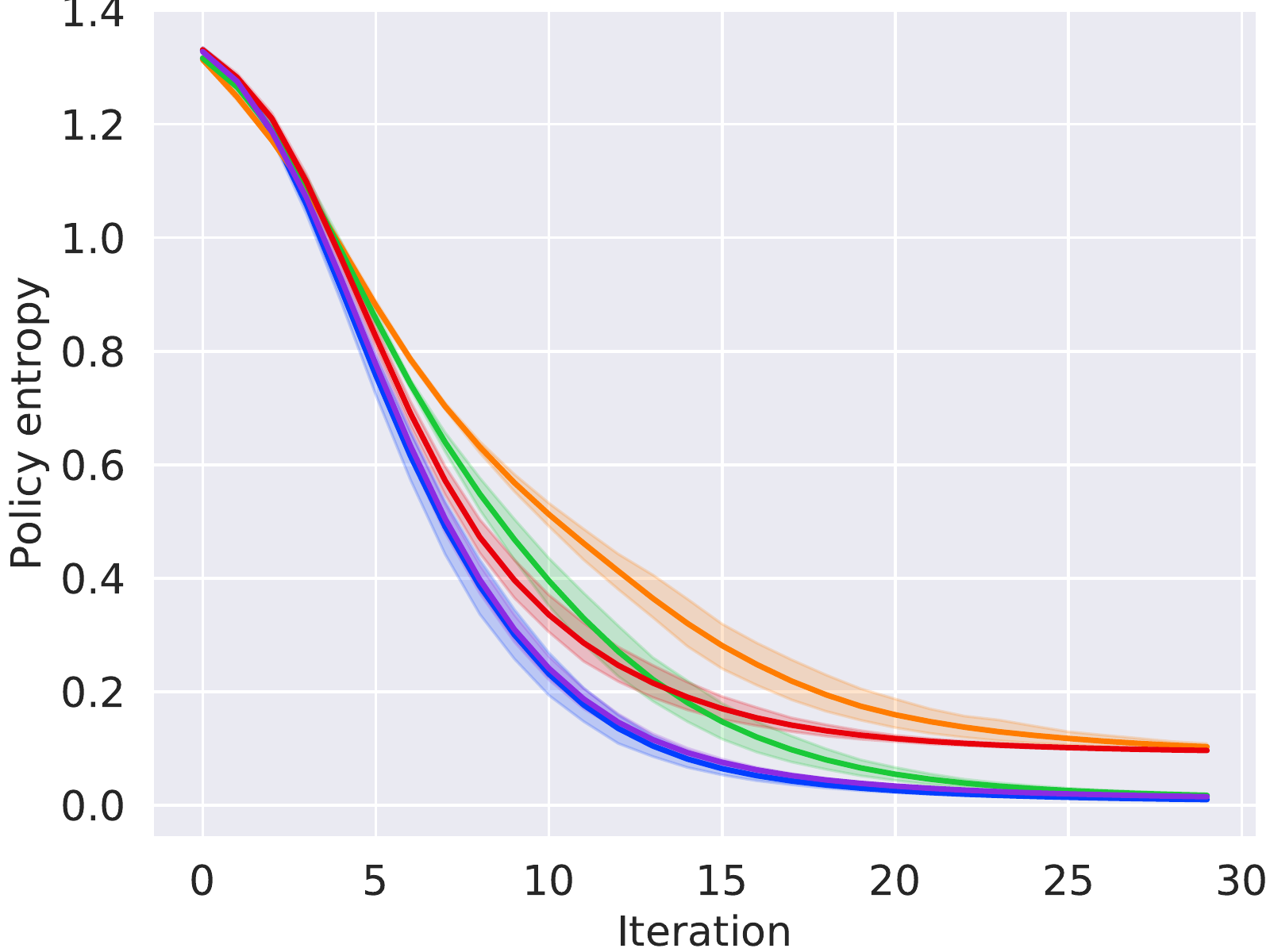}}
\caption{Policy entropy of all methods under different cases. The solid lines correspond to the mean and the shaded regions correspond to 95\% confidence interval over 5 runs.}
\label{fig.exp2.policy_entropy}
\end{figure*}

To summary, the choice of the state distribution in PG would cause the difference in the final performance. Large discrepancy between it and the interested state distribution would probably damage the performance. The stationary state distribution is usually a better choice than the initial state distribution because it reflects the states visited most frequently by the agent. The initial state distribution may omit important states or emphasis worthless ones. On the other hand, the approximate value function PG mainly influences the convergence speed if it can converge to the ground truth, otherwise it would also lead to worse performance. In practice, generally speaking, the true value function is inaccessible, therefore the more accurate the value estimates, the more perfect the policy gradient is, and the better the convergence speed and performance are. A suggestion from our experiment is that updating the value function more frequently that the policy. Besides, a larger learning rate is often preferred.

\section{Conclusion}\label{sec.conclusion}
In this paper, we group current RL algorithms by the direct and indirect taxonomy enlightened by the field of optimal control, where direct RL is defined as algorithms that solve the optimal policy by directly optimizing the expectation of accumulative future rewards using gradient descent methods while indirect RL is defined as algorithms that get optimal policy by indirectly solving the sufficient and necessary condition from Bellman's principle of optimality, i.e., the Bellman equation. We study policy gradient forms of direct and indirect RL, show that both of them can derive the actor-critic architecture and can be unified into a policy gradient with the approximate value function and the stationary state distribution, revealing the equivalence of direct and indirect RL. On that basis, we classify current mainstream RL algorithms by the direct and indirect taxonomy and compare it with other taxonomies including value-based and policy-based, model-based and model-free. Besides, we study the difference of the direct and indirect RL by simulations in a Gridworld task, suggesting experimentally the effect of different components in the PG, including the state distribution and the value function.

In the future, the findings of this work are promising to help to extend and reshape the current RL algorithms from its understanding of the behind learning mechanisms. For instance, we can design a brand new sufficient and necessary condition of the optimal policy in the indirect RL or establish a novel objective function in the direct RL to develop algorithms with better performance. In addition, as in the optimal control, the two families of algorithms can also be combined to eliminate their respective shortcomings.

\section*{acknowledgements}
We would like to acknowledge Mr. Zhengyu Liu for his valuable suggestions throughout this research. This work was supported by International Science \& Technology Cooperation Program of China under 2019YFE0100200, and Tsinghua University-Toyota Joint Research Center for AI Technology of Automated Vehicle.

\bibliography{main.bbl}

\providecommand{\url}[1]{\texttt{#1}}
\providecommand{\urlprefix}{}
\providecommand{\foreignlanguage}[2]{#2}
\providecommand{\Capitalize}[1]{\uppercase{#1}}
\providecommand{\capitalize}[1]{\expandafter\Capitalize#1}
\providecommand{\bibliographycite}[1]{\cite{#1}}
\providecommand{\bbland}{and}
\providecommand{\bblchap}{chap.}
\providecommand{\bblchapter}{chapter}
\providecommand{\bbletal}{et~al.}
\providecommand{\bbleditors}{editors}
\providecommand{\bbleds}{eds: }
\providecommand{\bbleditor}{editor}
\providecommand{\bbled}{ed.}
\providecommand{\bbledition}{edition}
\providecommand{\bbledn}{ed.}
\providecommand{\bbleidp}{page}
\providecommand{\bbleidpp}{pages}
\providecommand{\bblerratum}{erratum}
\providecommand{\bblin}{in}
\providecommand{\bblmthesis}{Master's thesis}
\providecommand{\bblno}{no.}
\providecommand{\bblnumber}{number}
\providecommand{\bblof}{of}
\providecommand{\bblpage}{page}
\providecommand{\bblpages}{pages}
\providecommand{\bblp}{}
\providecommand{\bblphdthesis}{Ph.D. thesis}
\providecommand{\bblpp}{}
\providecommand{\bbltechrep}{}
\providecommand{\bbltechreport}{Technical Report}
\providecommand{\bblvolume}{volume}
\providecommand{\bblvol}{Vol.}
\providecommand{\bbljan}{January}
\providecommand{\bblfeb}{February}
\providecommand{\bblmar}{March}
\providecommand{\bblapr}{April}
\providecommand{\bblmay}{May}
\providecommand{\bbljun}{June}
\providecommand{\bbljul}{July}
\providecommand{\bblaug}{August}
\providecommand{\bblsep}{September}
\providecommand{\bbloct}{October}
\providecommand{\bblnov}{November}
\providecommand{\bbldec}{December}
\providecommand{\bblfirst}{First}
\providecommand{\bblfirsto}{1st}
\providecommand{\bblsecond}{Second}
\providecommand{\bblsecondo}{2nd}
\providecommand{\bblthird}{Third}
\providecommand{\bblthirdo}{3rd}
\providecommand{\bblfourth}{Fourth}
\providecommand{\bblfourtho}{4th}
\providecommand{\bblfifth}{Fifth}
\providecommand{\bblfiftho}{5th}
\providecommand{\bblst}{st}
\providecommand{\bblnd}{nd}
\providecommand{\bblrd}{rd}
\providecommand{\bblth}{th}
\begin{thebibliography}{10}

\bibitem{levine2018learning}
S.~Levine, P.~Pastor, A.~Krizhevsky, J.~Ibarz, D.~Quillen, {\it Int. J. Rob.
  Res.} \textbf{2018}, {\it 37}, 421.

\bibitem{mnih2015human}
V.~Mnih, K.~Kavukcuoglu, D.~Silver, A.~A. Rusu, J.~Veness, M.~G. Bellemare,
  A.~Graves, M.~Riedmiller, A.~K. Fidjeland, G.~Ostrovski, S.~Petersen,
  C.~Beattie, A.~Sadik, I.~Antonoglou, H.~King, D.~Kumaran, D.~Wierstra,
  S.~Legg, D.~Hassabis, {\it Nature} \textbf{2015}, {\it 518}, 529.

\bibitem{vinyals2019grandmaster}
O.~Vinyals, I.~Babuschkin, W.~M. Czarnecki, M.~Mathieu, A.~Dudzik, J.~Chung,
  D.~H. Choi, R.~Powell, T.~Ewalds, P.~Georgiev, J.~Oh, D.~Horgan, M.~Kroiss,
  I.~Danihelka, A.~Huang, L.~Sifre, T.~Cai, J.~P. Agapiou, M.~Jaderberg, A.~S.
  Vezhnevets, R.~Leblond, T.~Pohlen, V.~Dalibard, D.~Budden, Y.~Sulsky,
  J.~Molloy, T.~L. Paine, C.~Gulcehre, Z.~Wang, T.~Pfaff, Y.~Wu, R.~Ring,
  D.~Yogatama, D.~Wünsch, K.~McKinney, O.~Smith, T.~Schaul, T.~Lillicrap,
  K.~Kavukcuoglu, D.~Hassabis, C.~Apps, D.~Silver, {\it Nature} \textbf{2019},
  {\it 575}, 350.

\bibitem{silver2016mastering}
D.~Silver, A.~Huang, C.~J. Maddison, A.~Guez, L.~Sifre, G.~v.~d. Driessche,
  J.~Schrittwieser, I.~Antonoglou, V.~Panneershelvam, M.~Lanctot, S.~Dieleman,
  D.~Grewe, J.~Nham, N.~Kalchbrenner, I.~Sutskever, T.~Lillicrap, M.~Leach,
  K.~Kavukcuoglu, T.~Graepel, D.~Hassabis, {\it Nature} \textbf{2016}, {\it
  529}, 484.

\bibitem{silver2017mastering}
D.~Silver, J.~Schrittwieser, K.~Simonyan, I.~Antonoglou, A.~Huang, A.~Guez,
  T.~Hubert, L.~Baker, M.~Lai, A.~Bolton, Y.~Chen, T.~Lillicrap, F.~Hui,
  L.~Sifre, G.~v.~d. Driessche, T.~Graepel, D.~Hassabis, {\it Nature}
  \textbf{2017}, {\it 550}, 354.

\bibitem{sutton2018reinforcement}
R.~S. Sutton, A.~G. Barto, {\it Reinforcement learning: An introduction}, MIT
  press, \textbf{2018}.

\bibitem{puterman1978modified}
M.~L. Puterman, M.~C. Shin, {\it Manag. Sci.} \textbf{1978}, {\it 24}, 1127.

\bibitem{sutton1988learning}
R.~S. Sutton, {\it Mach. Learn.} \textbf{1988}, {\it 3}, 9.

\bibitem{rummery1994line}
G.~A. Rummery, M.~Niranjan, {\it On-line Q-learning using connectionist
  systems}, {\it \bblvol{}~37}, University of Cambridge, Department of
  Engineering Cambridge, UK, \textbf{1994}.

\bibitem{watkins1992q}
C.~J. Watkins, P.~Dayan, {\it Mach. Learn.} \textbf{1992}, {\it 8}, 279.

\bibitem{van2016deep}
H.~Van~Hasselt, A.~Guez, D.~Silver, \bblin{} {\it Proc. the thirtieth AAAI
  Conference on Artificial Intelligence}, \textbf{2016}.

\bibitem{schaul2015prioritized}
T.~Schaul, J.~Quan, I.~Antonoglou, D.~Silver, \bblin{} {\it Proc. the 4th
  International Conference on Learning Representations}, \textbf{2016}.

\bibitem{wang2015dueling}
Z.~Wang, T.~Schaul, M.~Hessel, H.~Hasselt, M.~Lanctot, N.~Freitas, \bblin{}
  {\it Proc. the 33rd International Conference on Machine Learning}, PMLR, New
  York, New York, USA, \textbf{2016}, \bblp{}1995.

\bibitem{bellemare2017distributional}
M.~G. Bellemare, W.~Dabney, R.~Munos, \bblin{} {\it Proc. the 34th
  International Conference on Machine Learning}, PMLR, \textbf{2017},
  \bblp{}449.

\bibitem{hessel2018rainbow}
M.~Hessel, J.~Modayil, H.~v.~Hasselt, T.~Schaul, G.~Ostrovski, W.~Dabney,
  D.~Horgan, B.~Piot, M.~G. Azar, D.~Silver, \bblin{} {\it Proc. the
  Thirty-Second Conference on Artificial Intelligence}, {AAAI} Press,
  \textbf{2018}, \bblp{}3215.

\bibitem{munos2016safe}
R.~Munos, T.~Stepleton, A.~Harutyunyan, M.~Bellemare, \bblin{} {\it Proc.
  Advances in Neural Information Processing Systems 29}, Curran Associates,
  Inc., \textbf{2016}.

\bibitem{gu2016continuous}
S.~Gu, T.~Lillicrap, I.~Sutskever, S.~Levine, \bblin{} {\it Proc. The 33rd
  International Conference on Machine Learning}, PMLR, New York, New York, USA,
  \textbf{2016}, \bblp{}2829.

\bibitem{oh2017value}
J.~Oh, S.~Singh, H.~Lee, \bblin{} {\it Proc. Advances in Neural Information
  Processing Systems 30}, Curran Associates, Inc., \textbf{2017}.

\bibitem{holland2018effect}
G.~Z. Holland, E.~Talvitie, M.~Bowling, {\it CoRR} \textbf{2018}, {\it
  abs/1806.01825}.

\bibitem{azizzadenesheli2018sample}
K.~Azizzadenesheli, B.~Yang, W.~Liu, E.~Brunskill, Z.~C. Lipton, A.~Anandkumar,
  {\it CoRR} \textbf{2018}, {\it abs/1806.05780}.

\bibitem{williams1992simple}
R.~J. Williams, {\it Mach. Learn.} \textbf{1992}, {\it 8}, 229.

\bibitem{sutton2000policy}
R.~S. Sutton, D.~McAllester, S.~Singh, Y.~Mansour, \bblin{} {\it Proc. Advances
  in Neural Information Processing Systems 13}, MIT Press, \textbf{2000}.

\bibitem{konda2000actor}
V.~Konda, J.~Tsitsiklis, \bblin{} {\it Proc. Advances in Neural Information
  Processing Systems 13}, MIT Press, \textbf{2000}.

\bibitem{mnih2016asynchronous}
V.~Mnih, A.~P. Badia, M.~Mirza, A.~Graves, T.~Lillicrap, T.~Harley, D.~Silver,
  K.~Kavukcuoglu, \bblin{} {\it Proc. the 33rd International Conference on
  Machine Learning}, PMLR, New York, New York, USA, \textbf{2016}, \bblp{}1928.

\bibitem{horgan2018distributed}
D.~Horgan, J.~Quan, D.~Budden, G.~Barth{-}Maron, M.~Hessel, H.~v.~Hasselt,
  D.~Silver, \bblin{} {\it Proc. 6th International Conference on Learning
  Representations}, OpenReview.net, \textbf{2018}.

\bibitem{espeholt2018impala}
L.~Espeholt, H.~Soyer, R.~Munos, K.~Simonyan, V.~Mnih, T.~Ward, Y.~Doron,
  V.~Firoiu, T.~Harley, I.~Dunning, S.~Legg, K.~Kavukcuoglu, \bblin{} {\it
  Proc. the 35th International Conference on Machine Learning}, PMLR,
  \textbf{2018}, \bblp{}1407.

\bibitem{degris2012off}
T.~Degris, M.~White, R.~S. Sutton, {\it CoRR} \textbf{2012}, {\it
  abs/1205.4839}.

\bibitem{wang2016sample}
Z.~Wang, V.~Bapst, N.~Heess, V.~Mnih, R.~Munos, K.~Kavukcuoglu, N.~d.~Freitas,
  \bblin{} {\it Proc. 5th International Conference on Learning
  Representations}, OpenReview.net, \textbf{2017}.

\bibitem{gruslys2017reactor}
A.~Gruslys, W.~Dabney, M.~G. Azar, B.~Piot, M.~G. Bellemare, R.~Munos, \bblin{}
  {\it Proc. 6th International Conference on Learning Representations},
  OpenReview.net, \textbf{2018}.

\bibitem{gu2017interpolated}
S.~S. Gu, T.~Lillicrap, R.~E. Turner, Z.~Ghahramani, B.~Sch\"{o}lkopf,
  S.~Levine, \bblin{} {\it Proc. Advances in Neural Information Processing
  Systems 30}, Curran Associates, Inc., \textbf{2017}.

\bibitem{oh2018self}
J.~Oh, Y.~Guo, S.~Singh, H.~Lee, \bblin{} {\it Proc. the 35th International
  Conference on Machine Learning}, PMLR, \textbf{2018}, \bblp{}3878.

\bibitem{silver2014deterministic}
D.~Silver, G.~Lever, N.~Heess, T.~Degris, D.~Wierstra, M.~Riedmiller, \bblin{}
  {\it Proc. the 31st International Conference on Machine Learning}, PMLR,
  Bejing, China, \textbf{2014}, \bblp{}387.

\bibitem{lillicrap2015continuous}
T.~P. Lillicrap, J.~J. Hunt, A.~Pritzel, N.~Heess, T.~Erez, Y.~Tassa,
  D.~Silver, D.~Wierstra, \bblin{} {\it Proc. 4th International Conference on
  Learning Representations}, \textbf{2016}.

\bibitem{fujimoto2018addressing}
S.~Fujimoto, H.~v.~Hoof, D.~Meger, \bblin{} {\it Proc. the 35th International
  Conference on Machine Learning}, PMLR, \textbf{2018}, \bblp{}1587.

\bibitem{d4pg2018}
G.~Barth{-}Maron, M.~W. Hoffman, D.~Budden, W.~Dabney, D.~Horgan, D.~TB,
  A.~Muldal, N.~Heess, T.~P. Lillicrap, \bblin{} {\it Proc. the 6th
  International Conference on Learning Representations}, OpenReview.net,
  \textbf{2018}.

\bibitem{kakade2002natural}
S.~M. Kakade, \bblin{} {\it Proc. Advances in Neural Information Processing
  Systems 15}, MIT Press, \textbf{2002}.

\bibitem{schulman2015trust}
J.~Schulman, S.~Levine, P.~Abbeel, M.~Jordan, P.~Moritz, \bblin{} {\it Proc.
  the 32nd International Conference on Machine Learning}, PMLR, Lille, France,
  \textbf{2015}, \bblp{}1889.

\bibitem{wu2017scalable}
Y.~Wu, E.~Mansimov, R.~B. Grosse, S.~Liao, J.~Ba, \bblin{} {\it Proc. Advances
  in Neural Information Processing Systems 30}, Curran Associates, Inc.,
  \textbf{2017}.

\bibitem{schulman2017proximal}
J.~Schulman, F.~Wolski, P.~Dhariwal, A.~Radford, O.~Klimov, {\it CoRR}
  \textbf{2017}, {\it abs/1707.06347}.

\bibitem{haarnoja2017reinforcement}
T.~Haarnoja, H.~Tang, P.~Abbeel, S.~Levine, \bblin{} {\it Proc. the 34th
  International Conference on Machine Learning}, PMLR, \textbf{2017},
  \bblp{}1352.

\bibitem{haarnoja2018soft}
T.~Haarnoja, A.~Zhou, P.~Abbeel, S.~Levine, \bblin{} {\it Proc. the 35th
  International Conference on Machine Learning}, PMLR, \textbf{2018},
  \bblp{}1861.

\bibitem{duan2020addressing}
J.~Duan, Y.~Guan, Y.~Ren, S.~E. Li, B.~Cheng, {\it CoRR} \textbf{2020}, {\it
  abs/2001.02811}.

\bibitem{salimans2017evolution}
T.~Salimans, J.~Ho, X.~Chen, I.~Sutskever, {\it CoRR} \textbf{2017}, {\it
  abs/1703.03864}.

\bibitem{mania2018simple}
H.~Mania, A.~Guy, B.~Recht, \bblin{} {\it Proc. Advances in Neural Information
  Processing Systems 31}, Curran Associates, Inc., \textbf{2018}.

\bibitem{deisenroth2011pilco}
M.~P. Deisenroth, C.~E. Rasmussen, \bblin{} {\it Proc. the 28th International
  Conference on Machine Learning}, Omnipress, \textbf{2011}, \bblp{}465.

\bibitem{DREAMER}
D.~Hafner, T.~P. Lillicrap, J.~Ba, M.~Norouzi, \bblin{} {\it Proc. the 8th
  International Conference on Learning Representations}, OpenReview.net,
  \textbf{2020}.

\bibitem{CURL}
M.~Laskin, A.~Srinivas, P.~Abbeel, \bblin{} {\it Proc. the 37th International
  Conference on Machine Learning}, {PMLR}, \textbf{2020}, \bblp{}5639.

\bibitem{AUGMENTEDATA}
M.~Laskin, K.~Lee, A.~Stooke, L.~Pinto, P.~Abbeel, A.~Srinivas, \bblin{} {\it
  Proc. Advances in Neural Information Processing Systems 33}, \textbf{2020}.

\bibitem{white1992handbook}
D.~White, D.~Sofge, {\it Handbook of Intelligent Control: Neural, Fuzzy, and
  Adaptive Approaches},  \bblof{} {\it VNR computer library}, Van Nostrand
  Reinhold, \textbf{1992}.

\bibitem{prokhorov1997adaptive}
D.~V. {Prokhorov}, D.~C. {Wunsch}, {\it IEEE Trans. Neural Netw. Learn. Syst.}
  \textbf{1997}, {\it 8}, 997.

\bibitem{heess2015learning}
N.~Heess, G.~Wayne, D.~Silver, T.~Lillicrap, T.~Erez, Y.~Tassa, \bblin{} {\it
  Proc. Advances in Neural Information Processing Systems 28}, Curran
  Associates, Inc., \textbf{2015}.

\bibitem{feinberg2018model}
V.~Feinberg, A.~Wan, I.~Stoica, M.~I. Jordan, J.~E. Gonzalez, S.~Levine, {\it
  CoRR} \textbf{2018}, {\it abs/1803.00101}.

\bibitem{levine2013guided}
S.~Levine, V.~Koltun, \bblin{} {\it Proc. the 30th International Conference on
  Machine Learning}, PMLR, Atlanta, Georgia, USA, \textbf{2013}, \bblp{}~1.

\bibitem{buckman2018sample}
J.~Buckman, D.~Hafner, G.~Tucker, E.~Brevdo, H.~Lee, \bblin{} {\it Proc.
  Advances in Neural Information Processing Systems 31}, Curran Associates,
  Inc., \textbf{2018}.

\bibitem{kurutach2018model}
T.~Kurutach, I.~Clavera, Y.~Duan, A.~Tamar, P.~Abbeel, \bblin{} {\it Proc. the
  6th International Conference on Learning Representations}, OpenReview.net,
  \textbf{2018}.

\bibitem{ha2018recurrent}
D.~Ha, J.~Schmidhuber, \bblin{} {\it Proc. Advances in Neural Information
  Processing Systems 31}, Curran Associates, Inc., \textbf{2018}.

\bibitem{racaniere2017imagination}
S.~Racani\`{e}re, T.~Weber, D.~Reichert, L.~Buesing, A.~Guez,
  D.~Jimenez~Rezende, A.~Puigdom\`{e}nech~Badia, O.~Vinyals, N.~Heess, Y.~Li,
  R.~Pascanu, P.~Battaglia, D.~Hassabis, D.~Silver, D.~Wierstra, \bblin{} {\it
  Proc. Advances in Neural Information Processing Systems 30}, Curran
  Associates, Inc., \textbf{2017}.

\bibitem{janner2019trust}
M.~Janner, J.~Fu, M.~Zhang, S.~Levine, \bblin{} {\it Proc. Advances in Neural
  Information Processing Systems 32}, Curran Associates, Inc., \textbf{2019}.

\bibitem{ziebart2010modeling}
B.~D. Ziebart, \bblphdthesis{}, \textbf{2018}.

\bibitem{konda2003onactor}
V.~R. Konda, J.~N. Tsitsiklis, {\it SIAM J. Control Optim.} \textbf{2003}, {\it
  42}, 1143.

\bibitem{schulman2017equivalence}
J.~Schulman, P.~Abbeel, X.~Chen, {\it CoRR} \textbf{2017}, {\it
  abs/1704.06440}.

\bibitem{nachum2017bridging}
O.~Nachum, M.~Norouzi, K.~Xu, D.~Schuurmans, \bblin{} {\it Proc. Advances in
  Neural Information Processing Systems 30}, Curran Associates, Inc.,
  \textbf{2017}.

\bibitem{nachum2017trust}
O.~Nachum, M.~Norouzi, K.~Xu, D.~Schuurmans, \bblin{} {\it Proc. 6th
  International Conference on Learning Representations}, OpenReview.net,
  \textbf{2018}.

\bibitem{ross1996stochastic}
S.~M. Ross, J.~J. Kelly, R.~J. Sullivan, W.~J. Perry, D.~Mercer, R.~M. Davis,
  T.~D. Washburn, E.~V. Sager, J.~B. Boyce, V.~L. Bristow, {\it Stochastic
  processes}, {\it \bblvol{}~2}, Wiley New York, \textbf{1996}.

\bibitem{thomas2014bias}
P.~Thomas, \bblin{} {\it Proc. the 31st International Conference on Machine
  Learning}, PMLR, Bejing, China, \textbf{2014}, \bblp{}441.

\bibitem{bellman1966dynamic}
R.~Bellman, {\it Science} \textbf{1966}, {\it 153}, 34.

\bibitem{duan2019deep}
J.~Duan, Z.~Liu, S.~E. Li, Q.~Sun, Z.~Jia, B.~Cheng, {\it CoRR} \textbf{2019},
  {\it abs/1911.11397}.

\bibitem{duan2019generalized}
J.~Duan, S.~E. Li, Z.~Liu, M.~Bujarbaruah, B.~Cheng, {\it CoRR} \textbf{2019},
  {\it abs/1909.05402}.

\bibitem{imani2018off}
E.~Imani, E.~Graves, M.~White, \bblin{} {\it Proc. Advances in Neural
  Information Processing Systems 31}, Curran Associates, Inc., \textbf{2018}.

\bibitem{zhang2019generalized}
S.~Zhang, W.~Boehmer, S.~Whiteson, \bblin{} {\it Proc. Advances in Neural
  Information Processing Systems 32}, Curran Associates, Inc., \textbf{2019}.

\bibitem{bertsekas2000gradient}
D.~P. Bertsekas, J.~N. Tsitsiklis, {\it SIAM J. Optim.} \textbf{2000}, {\it
  10}, 627.

\bibitem{bertsekas1996neuro}
D.~P. Bertsekas, J.~N. Tsitsiklis, {\it Neuro-dynamic programming}, Athena
  Scientific, \textbf{1996}.

\bibitem{hendrycks2016gaussian}
D.~Hendrycks, K.~Gimpel, {\it CoRR} \textbf{2016}, {\it abs/1606.08415}.

\end{thebibliography}

\end{document}